\documentclass{article} 

\usepackage{iclr2027_conference,times}

\usepackage{amsmath,amsfonts,bm}

\def\eqref#1{equation~\ref{#1}}

\def\1{\bm{1}}

\DeclareMathAlphabet{\mathsfit}{\encodingdefault}{\sfdefault}{m}{sl}
\SetMathAlphabet{\mathsfit}{bold}{\encodingdefault}{\sfdefault}{bx}{n}

\usepackage{hyperref}
\usepackage{url}

\usepackage{amsmath,amssymb,amsthm,mathtools,bm}
\usepackage{graphicx,booktabs,tabularx,multirow,xcolor}
\usepackage{hyperref,url}
\newcommand{\method}{Tangent-SBM}

\newtheorem{theorem}{Theorem}
\newtheorem{proposition}[theorem]{Proposition}

\title{Tangent Schr\"odinger Bridge Matching:
Learning Stochastic Transport with Mechanistic Sensitivities}

\author{
Jowaria Khan \\
University of Michigan \\
\texttt{jowaria@umich.edu}
\And
Elizabeth Bondi-Kelly \\
University of Michigan \\
\texttt{ecbk@umich.edu}
}

\iclrfinalcopy 
\begin{document}

\maketitle

\begin{abstract}
Predicting how stochastic systems respond to changes
in viscosity, reaction rates, or external forces
requires costly simulations, motivating reusable
learned models.
Yet matching observed outcome distributions does
not ensure accurate intervention responses.
We introduce \textbf{Tangent Schr\"odinger Bridge Matching
(Tangent-SBM)}, which learns stochastic transports
from endpoint observations and mechanistic sensitivities.
It propagates parameter derivatives alongside
trajectories and supervises them against supplied targets.
For average-response targets, single-rollout squared
error also penalizes response variability; our objective
uses two independent rollouts to match the mean
without this additional penalty.
We establish conditions under which sensitivity
accuracy bounds finite-change prediction error
and decision regret.
Across Gaussian, stochastic double-well, PDEBench
reaction--diffusion, and stochastic Navier--Stokes
systems, Tangent-SBM improves sensitivity and
finite-change prediction over matched conditional-bridge
baselines while maintaining comparable endpoint
and distributional accuracy.
Controls examine target correctness, response
objectives, and simulator-budget allocation.
To test decision usefulness, we evaluate calibrated
viscosity selection in Navier--Stokes:
Tangent-SBM reduces tracking error relative to
taking no action on every evaluated task.
\end{abstract}

\section{Introduction}
\label{sec:introduction}

Scientific simulators use mathematical models to predict
how a system evolves; for example, how a fluid flows
at a given viscosity.
An intervention changes a parameter or operating condition.
In stochastic systems, random forcing or fluctuations
produce different outcomes even with the same inputs.
Comparing interventions can therefore require many
numerical solves for each candidate and initial condition,
making repeated optimization or decision-making costly.
A reusable surrogate shifts computation to training
and can then evaluate new candidates more cheaply
\citep{oleary2024dino,luo2025optimization}.
For stochastic systems, such a surrogate should represent the distribution of possible outcomes, not only their mean.


Accurate predictions at observed parameter values do not,
however, ensure accurate responses to parameter changes.
For example, a surrogate used to choose an emission rate must predict
how adjusting that rate changes pollutant concentrations.
A model can fit the available observations while
recommending an ineffective intervention.


For small parameter changes, this response is described
by a derivative, or \emph{sensitivity}.
In a stochastic system, it may describe an individual
trajectory's response with random inputs held fixed,
or the average response across trajectories.
Physical models and simulators supply such derivatives
through analytic expressions, differentiation, or
numerical perturbations. We call this information
\emph{mechanistic sensitivities}.

Schr\"odinger bridges learn stochastic processes
connecting prescribed endpoint distributions while
remaining close to a reference process
\citep{leonard2014survey}.
Neural and conditional formulations make them useful
for learning parameter-dependent outcome distributions
\citep{debortoli2021diffusion,shi2023dsbm,
tong2024simulationfree,shi2022conditional}.
Yet endpoint agreement at finitely many parameter
values leaves their local responses undetermined.
Derivative-informed scientific learning offers additional
supervision
\citep{czarnecki2017sobolev,oleary2024dino},
raising the question:
\emph{How should mechanistic sensitivities constrain
a learned stochastic transport?}

We introduce \textbf{Tangent Schr\"odinger Bridge
Matching (Tangent-SBM)}, which propagates parameter
sensitivities alongside learned trajectories and
supervises their terminal responses.
These responses can vary across trajectories:
for example, changing emissions may affect pollutant
concentrations differently under different wind
conditions.
If supervision specifies only the average effect,
forcing every trajectory to match it also suppresses
response variability.
Our two-rollout objective instead matches the average
without this additional penalty.

Because sensitivities and endpoint observations come
from the same simulator, we compare response labels
with spending that budget on more outcome samples.
We also test direct predictors when only a mean,
such as average pollutant concentration, is needed;
Tangent-SBM additionally models the distribution
of possible outcomes.


Our contributions are:
\begin{itemize}
\item \textbf{Response-supervised stochastic transport.}
We introduce Tangent-SBM with distinct objectives
for individual and average response targets.

    \item \textbf{From sensitivities to interventions.}
    We show why matching observed outcomes can leave
    parameter responses incorrect, and establish when
    accurate sensitivities lead to accurate predictions
    of intervention effects and better decisions.

    \item \textbf{Controlled evaluation across four systems.}
    Gaussian, double-well, reaction--diffusion, and
    stochastic Navier--Stokes experiments assess response
    accuracy, endpoint fidelity, and, for stochastic
    systems, distributional fit. Controls test target
    correctness and simulation-budget allocation;
    calibrated intervention selection tests downstream
    usefulness.
\end{itemize}

\begin{figure}[t]
    \centering
\includegraphics[width=1.0\linewidth]
        {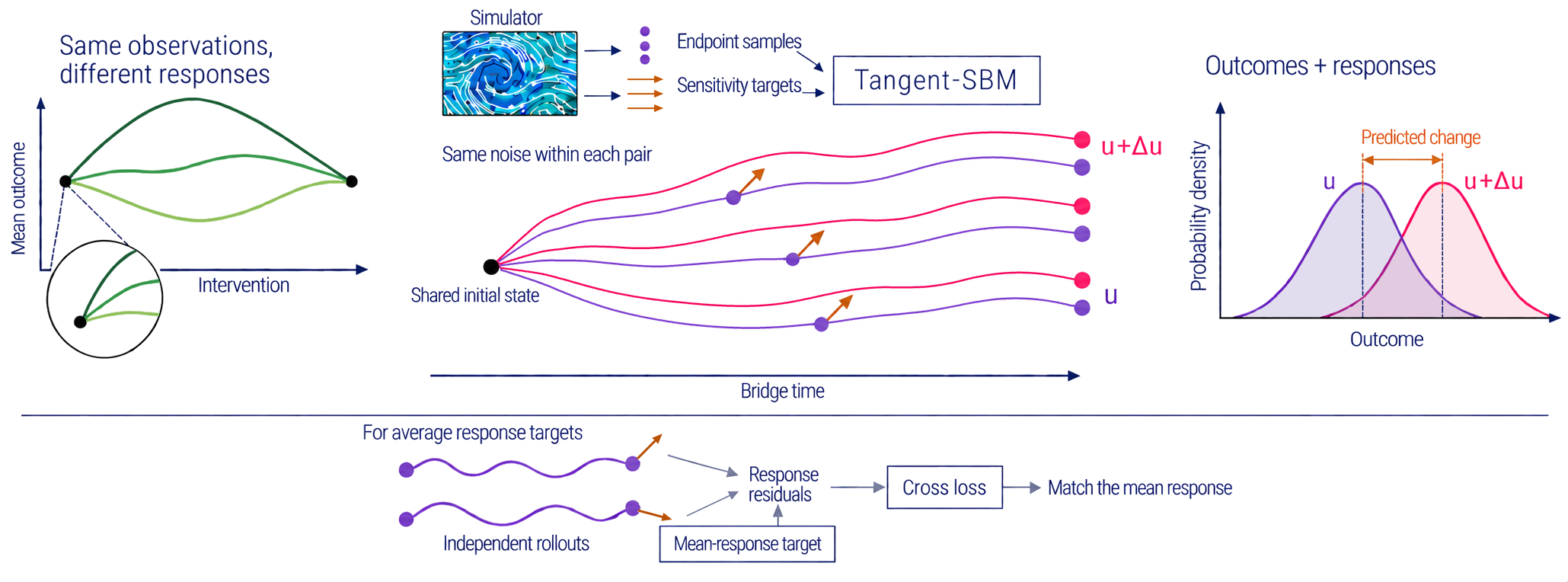
    }
    \caption{
Tangent-SBM learns outcome distributions and
intervention responses by supervising trajectory
sensitivities. Independent rollout pairs match
mean-response targets without penalizing response variability.
    }
    \label{fig:framework}
\end{figure}

\section{Related Work}
\label{sec:related_work}

\paragraph{Schr\"odinger bridges and stochastic generative models.}
Schr\"odinger bridges connect endpoint distributions through a process close to a reference diffusion \citep{leonard2014survey}, linking stochastic transport with entropy-regularized optimal transport \citep{cuturi2013sinkhorn,peyre2019computational}. Neural solvers include deep generative bridges, diffusion-based iterative fitting, and DSBM \citep{wang2021deep,debortoli2021diffusion,shi2023dsbm}, with extensions to conditional sampling and simulation-free training \citep{shi2022conditional,tong2024simulationfree}. Generalized and twisted formulations incorporate trajectory costs or potentials \citep{liu2022deepgsb,liu2023gsbm,noble2026tsbm}. Related diffusion, flow-matching, stochastic-interpolant, and neural-SDE models provide alternative ways to learn distributions and stochastic dynamics \citep{ho2020denoising,song2021score,lipman2023flow,albergo2023interpolants,kidger2021sde}. Tangent-SBM adds mechanistic sensitivity supervision
to constrain intervention responses beyond endpoint fitting.


\paragraph{Scientific surrogates and derivative-informed learning.}
Physics-informed learning incorporates governing
equations into training
\citep{raissi2019physics,karniadakis2021physics},
while standard deterministic neural operators map
input functions to output fields
\citep{lu2021deeponet,li2021fourier,kovachki2023operator}. Conditional stochastic bridges instead model a
distribution of possible outcomes for the same input. Physical constraints also inform diffusion-based field reconstruction \citep{li2025physicsaligned}. Explicit derivative supervision is central to Sobolev training and Jacobian matching \citep{czarnecki2017sobolev,srinivas2018jacobian}, and derivative-informed neural operators apply this principle to scientific surrogates \citep{oleary2024dino}. Subsequent work demonstrates its use in optimization under uncertainty and develops derivative-informed Fourier neural operators \citep{luo2025optimization,yao2025difno}. Tangent-SBM distinguishes individual and mean
responses in stochastic transports.
Deterministic controls test response supervision
beyond bridges; stochastic experiments assess
both outcome distributions and intervention responses.


\paragraph{Stochastic sensitivities and response objectives.}
Automatic differentiation and SDE sensitivity methods provide tools for differentiating numerical stochastic models \citep{baydin2018autodiff,li2020scalable}. The training objective must additionally reflect the supplied target: regressing each random response against a mean target with squared error penalizes response variance alongside mean error. We identify this effect and use independent rollouts to isolate the mean-error term. The construction follows the established double-sampling principle for squared expected residuals \citep{baird1995residual}; our contribution is its formulation and evaluation within mechanistic response supervision for stochastic bridges. 


\section{Tangent Schr\"odinger Bridge Matching}
\label{sec:method}

Tangent-SBM trains stochastic transports to match
observed outcomes and their responses to
interventions.
The key distinction is whether sensitivity targets
describe individual trajectories or only their
average: these require different supervision
objectives.
Figure~\ref{fig:framework} summarizes the framework.


\subsection{Conditional stochastic transport}
\label{sec:setting}

Let $c$ denote the context, including the initial state,
and $u\in\mathbb{R}^p$ the intervention.
The physical outcome $Y^*(c,u)$ has conditional mean
$m^*(c,u)=\mathbb{E}[Y^*(c,u)\mid c,u]$.
Our model generates $Y_\theta=h(X_T,u)$ through
\begin{equation}
dX_t=b_\theta(X_t,u,t)\,dt+\sigma_t\,dW_t,
\qquad X_0=x_0(c),
\label{eq:sde}
\end{equation}
where $b_\theta$ is the learned drift, $W_t$ is Brownian motion, a continuous random process
with independent Gaussian increments, and $h$ returns the predicted
quantity, such as the terminal field.
Write $m_\theta(c,u)=\mathbb{E}[Y_\theta\mid c,u]$.

We fit the transport using conditional DSBM
\citep{shi2022conditional,shi2023dsbm}.
Context dependence is suppressed in the drift notation.
We assume time-dependent additive noise $\sigma_t$
and an intervention-independent initial state.
Bridge time parameterizes the transport and need
not coincide with physical simulation time.

\subsection{Tracking trajectory responses}
\label{sec:tangent}

For a unit intervention direction $v$, define
the directional tangent
\begin{equation}
R_t^v=
\left.\frac{d}{d\epsilon}
X_t(u+\epsilon v;W)\right|_{\epsilon=0}.
\label{eq:tangent_definition}
\end{equation}
Holding the Brownian path fixed isolates the
effect of the intervention.
Differentiating \eqref{eq:sde} gives
\begin{equation}
dR_t^v=
\left[
\partial_x b_\theta(X_t,u,t)R_t^v
+\partial_u b_\theta(X_t,u,t)v
\right]dt,
\qquad R_0^v=0.
\label{eq:tangent}
\end{equation}
The two terms capture the effects of changes
in the state and the intervention, respectively.
The terminal response is
\begin{equation}
Z_\theta(c,u,v;W)
=\partial_x h(X_T,u)R_T^v+\partial_u h(X_T,u)v.
\label{eq:z}
\end{equation}
For terminal-state prediction, $h(X_T,u)=X_T$
and $Z_\theta=R_T^v$.

Jacobians collect partial derivatives of vector-valued
functions. Automatic differentiation computes
Jacobian--vector products without forming these
matrices explicitly \citep{pearlmutter1994fast,baydin2018autodiff},
following pathwise stochastic sensitivity analysis
\citep{glasserman1991gradient,li2020scalable}.
Our implementation differentiates the discretized
rollout.
Under the conditions in Appendix~\ref{app:regularity},
\begin{equation}
\mathbb{E}[Z_\theta\mid c,u,v]=J_\theta(c,u)v,
\qquad J_\theta(c,u)=\partial_u m_\theta(c,u).
\label{eq:mean_response}
\end{equation}
Thus, $Z_\theta$ is one trajectory's response,
whereas $J_\theta v$ is the response of the
mean prediction.

\subsection{Supervising individual and average responses}
\label{sec:response_objectives}

Derivative supervision constrains parameter dependence
in addition to output values
\citep{czarnecki2017sobolev,oleary2024dino}.
For stochastic outputs, the loss must distinguish
individual-response targets from mean-response targets.

\paragraph{Realization-level targets.}
Let $Z^*(c,u,v;\xi)$ be the physical response
with simulator randomness $\xi$ held fixed.
Comparing individual responses requires a specified,
$\theta$-independent pairing $\pi$ of model noise $W$
and simulator noise $\xi$, conditional on $(c,u,v)$.
Under this pairing,
\begin{equation}
\widehat\ell_{\mathrm{path}}
=\|Z_\theta(c,u,v;W)-Z^*(c,u,v;\xi)\|_2^2.
\label{eq:path_loss}
\end{equation}
Writing $\bar z_\theta$ and $\bar z^*$ for their
conditional means, Jensen's inequality gives
\begin{equation}
\|\bar z_\theta-\bar z^*\|_2^2
\leq
\mathbb{E}_{\pi}[\widehat\ell_{\mathrm{path}}\mid c,u,v].
\label{eq:path_mean_bound}
\end{equation}
The loss depends on the pairing.
When only a mean target is supplied, or no
trajectory pairing is justified, we instead
use the following objective.

\paragraph{Conditional-mean targets.}
Let $\bar z^*=J^*(c,u)v$, where
$J^*=\partial_u m^*$, and
$\bar z_\theta=\mathbb{E}[Z_\theta\mid c,u,v]$.
Regressing each response against the mean target gives
\begin{equation}
\mathbb{E}\|Z_\theta-\bar z^*\|_2^2
=
\|\bar z_\theta-\bar z^*\|_2^2
+\mathbb{E}\|Z_\theta-\bar z_\theta\|_2^2,
\label{eq:variance_penalty}
\end{equation}
with expectations conditional on $(c,u,v)$.
The second term penalizes response variability
although the target specifies only the mean.
Our Sobolev-DSBM (single-rollout) baseline uses
this squared-loss formulation.
Tangent-SBM instead uses two conditionally independent
rollouts at the same $(c,u,v)$, with responses
$Z_\theta^{(1)}$ and $Z_\theta^{(2)}$, to isolate
the mean-response error:
\begin{align}
\widehat\ell_{\mathrm{mean}}
&=\left\langle
Z_\theta^{(1)}-\bar z^*,
Z_\theta^{(2)}-\bar z^*
\right\rangle,
\label{eq:cross}\\
\mathbb{E}[\widehat\ell_{\mathrm{mean}}\mid c,u,v]
&=\|\bar z_\theta-\bar z^*\|_2^2.
\label{eq:cross_expectation}
\end{align}
This independent-sampling construction
\citep{baird1995residual} removes the additional
variance penalty.
Gradients pass through both responses.
Individual estimates may be negative, but their
expectation is the nonnegative mean-response error.
Appendix~\ref{app:meanloss} gives the derivation;
Appendix~\ref{app:noisytarget} treats bias and
sampling error in the supplied targets, which
are held fixed in these identities.

\subsection{Training with response supervision}
\label{sec:training}

We augment the current DSBM stage's drift-matching
loss \citep{shi2023dsbm}:
\begin{equation}
\mathcal L(\theta)
=\mathcal L_{\mathrm{BM}}(\theta)
+\lambda_{\mathrm{sens}}
\mathbb{E}_{(c,u,v)\sim\nu}
\left[
\mathbb{E}[\widehat\ell_{\mathrm{sens}}\mid c,u,v]
\right].
\label{eq:objective}
\end{equation}
Here $\nu$ specifies the supervision conditions
and directions; the inner expectation averages
rollout randomness and, where applicable,
the specified pairing.
We choose $\widehat\ell_{\mathrm{path}}$ or
$\widehat\ell_{\mathrm{mean}}$ according to the target.
Each response update samples $(c,u,v)$, propagates
model states and tangents, and evaluates this loss.
Both Tangent-SBM and the conditional baseline
receive $u$; response targets provide the
additional information.

Response supervision uses \emph{free rollouts}:
the learned model generates its terminal state
without a prescribed endpoint.
Unlike trajectories constrained to reach sampled
endpoints, these measure the response of the
sampler used for prediction.
Appendix~\ref{app:objective_ablations} compares
the two constructions.

Each direction adds a tangent array with the same
dimensions as the state and additional derivative
computations; mean supervision also requires
independent replicas.
Matching simulator-call budgets therefore does
not imply matching training cost.
The added loss balances endpoint and response fit
and does not generally preserve the original
DSBM solution.
\section{From local sensitivities to finite interventions}
\label{sec:theory}
We show why endpoint agreement leaves intervention
responses undetermined and establish when sensitivity
accuracy controls finite-change prediction error
and decision regret.
Fix the context $c$ and write $m_\theta(u)$ and $m^*(u)$
for the model and physical mean outcomes, with
Jacobians $J_\theta$ and $J^*$.
Proofs appear in Appendix~\ref{app:proofs}.

\paragraph{What endpoint observations leave undetermined.}
\begin{proposition}[Endpoint ambiguity]
\label{prop:nonid}
For finitely many distinct interventions
$u_1,\ldots,u_k\in\mathbb{R}^p$, there exist smooth
families of conditional distributions with finite
means that agree at every $u_i$ but have arbitrarily
different mean Jacobians at any chosen $u_i$.
\end{proposition}

Thus, endpoint agreement alone does not identify
intervention responses without further assumptions.
Sensitivity supervision supplies information about
this otherwise unconstrained parameter dependence.

\paragraph{From sensitivities to finite changes.}
\begin{proposition}[Finite-intervention error]
\label{thm:finite}
Let $\delta=rv$, with $r\geq0$ and $\|v\|_2=1$,
and suppose $m_\theta,m^*$ are continuously
differentiable on a neighborhood of the segment
$\{u+s\delta:s\in[0,1]\}$.
For the mean changes
$D_\theta=m_\theta(u+\delta)-m_\theta(u)$ and
$D^*=m^*(u+\delta)-m^*(u)$,
\begin{equation}
\|D_\theta-D^*\|_2^2
\leq r^2\int_0^1
\|[J_\theta(u+s\delta)-J^*(u+s\delta)]v\|_2^2\,ds.
\label{eq:finitebound}
\end{equation}
\end{proposition}

The bound requires accurate sensitivities along
the intervention segment, not only at its starting
point.
If $\varepsilon_{\mathrm{fit}}$ and
$\varepsilon_{\mathrm{target}}$ denote root-mean-square
directional errors along this segment, respectively
between the model and supplied targets and between
those targets and the true sensitivities, then
\begin{equation}
\|D_\theta-D^*\|_2
\leq r(\varepsilon_{\mathrm{fit}}
+\varepsilon_{\mathrm{target}}).
\label{eq:targetbound}
\end{equation}
Here the segment average is over $s\in[0,1]$.
Both fitting accuracy and target quality therefore
matter.

Appendix~\ref{app:coverage} relates sampled directions
and finite collocation coverage; the placement of
response-supervision points, to accuracy along
the segment under additional assumptions.
Small empirical training loss alone does not
establish this coverage.

\paragraph{Consequences for intervention selection.}
The following perturbation bound connects accurate
response predictions to useful decisions
\citep{bonnans2000perturbation}.

\begin{proposition}[Intervention-selection error]
\label{cor:decision}
Let $\mathcal A$ be a finite action set with true
mean changes $g(a)$ and predictions $\widehat g(a)$
satisfying
$\sup_{a\in\mathcal A}\|\widehat g(a)-g(a)\|_2\leq\eta$.
For a desired change $y$, let $\widehat a$ minimize
$\|\widehat g(a)-y\|_2$ and let $a^*$ minimize
$\|g(a)-y\|_2$ over $\mathcal A$. Then
\begin{equation}
\|g(\widehat a)-y\|_2-\|g(a^*)-y\|_2
\leq2\eta.
\label{eq:tracking_regret}
\end{equation}
For squared tracking loss
$\ell(a)=\|g(a)-y\|_2^2$, it follows that
\begin{equation}
\ell(\widehat a)-\ell(a^*)
\leq4\eta\|g(a^*)-y\|_2+4\eta^2.
\label{eq:regret}
\end{equation}
\end{proposition}

For model finite-change predictions,
Equation~\ref{eq:targetbound} bounds $\eta$ by the
largest action-wise value of
$r_a(\varepsilon_{\mathrm{fit},a}
+\varepsilon_{\mathrm{target},a})$.
For the linear predictions used in our
decision study, linearization
errors must also be included
(Appendix~\ref{app:decision}).
The guarantee bounds the tracking error of the selected action, even when that action differs from the oracle's choice.
\section{Experiments}
\label{sec:experiments}

We test whether mechanistic supervision improves
intervention responses beyond endpoint fitting,
which aspects of supervision matter, and whether
the learned responses support useful decisions.

\subsection{Systems and evaluation}
\label{sec:experimental_setup}

Each benchmark predicts an outcome from an initial
state and an intervention. Endpoint examples
supervise outcomes; sensitivity targets supervise
how those outcomes change with the intervention.
We use two low-dimensional systems and two
spatial-field systems.

\paragraph{Conditional Gaussian.}
We construct a two-dimensional Gaussian outcome:
\begin{equation}
Y^*(x_0,u)
=Ax_0+B_1u+B_2u^2+B_3\sin(\pi u)+\epsilon,
\qquad
\epsilon\sim\mathcal{N}(0,\Sigma).
\label{eq:exp_gaussian}
\end{equation}
Here $x_0$ is the initial state, $u$ is the scalar
intervention, $A,B_1,B_2,B_3$ are fixed coefficients,
and $\epsilon$ is Gaussian noise with covariance
$\Sigma$.
Endpoint training uses $u\in\{-1,0,1\}$, where
the sinusoidal term vanishes but its derivative
does not. The analytic sensitivity therefore
reveals parameter dependence absent from the
training outcomes, providing a controlled
example of Proposition~\ref{prop:nonid}.

\paragraph{Stochastic double well.}
We use a double-well Langevin system
\citep{pavliotis2014stochastic}:
\begin{equation}
dX_\tau=
[X_\tau-X_\tau^3+0.30\tanh(u)]\,d\tau
+0.45\,dB_\tau.
\label{eq:exp_doublewell}
\end{equation}
Here $X_\tau$ is the particle position,
$B_\tau$ is standard Brownian motion, and
$\tau$ denotes physical time, distinct from
bridge time.
The intervention $u$ tilts the potential,
favoring one of its two wells.
We predict the terminal distribution from
the initial position and intervention.
Targets average simulator trajectory derivatives
to describe the response of the mean outcome.

\paragraph{PDEBench reaction--diffusion.}
The PDEBench simulator \citep{takamoto2022pdebench}
evolves two interacting fields:
\begin{equation}
\begin{aligned}
\partial_\tau q_1
&=q_1-q_1^3-k-q_2+D_1\Delta q_1,\\
\partial_\tau q_2
&=q_1-q_2+D_2\Delta q_2.
\end{aligned}
\label{eq:exp_pdebench}
\end{equation}
Here $q_1,q_2$ are the evolving fields,
$D_1,D_2$ are diffusion coefficients,
$k$ is a reaction parameter, and $\Delta$
is the spatial Laplacian.
We predict both terminal fields on a
$128\times128$ grid from their initial configuration.
The intervention vector $u=(u_1,u_2,u_3)$ sets
$D_1=D_1^0 2^{u_1}$, $D_2=D_2^0 2^{u_2}$,
and $k=k^0 2^{u_3}$, where superscript $0$
denotes reference values.
Endpoint training uses $\{-1,0,1\}^3$;
targets estimate directional sensitivities
by perturbing these parameters.
PDEBench tests tangent-based response supervision
in high-dimensional spatial fields.
Because the simulator is deterministic given
its inputs, this benchmark evaluates response
learning within the bridge framework rather
than the need for stochastic outcome modeling.

\paragraph{Stochastic Navier--Stokes (SNS).}
We use the SPDEBench solver \citep{zhu2025spdebench}:
\begin{equation}
d\omega=
[-(\mathbf v\cdot\nabla)\omega
+\nu(u)\Delta\omega+f]\,d\tau+d\mathcal W_\tau,
\qquad \nu(u)=\nu_0 2^u.
\label{eq:exp_sns}
\end{equation}
Here $\omega$ is vorticity, measuring local fluid
rotation, $\mathbf v$ is the incompressible velocity,
and $\nu(u)$ is viscosity with reference value
$\nu_0$.
The operator $\nabla$ is the spatial gradient;
the first two drift terms describe advection
and diffusion.
The terms $f$ and $\mathcal W_\tau$ supply
deterministic and spatially correlated random
forcing.
The model predicts terminal vorticity from
the initial field and intervention $u$, without
observing the random forcing.
A unit increase in $u$ doubles viscosity.
Mean-response targets average paired simulator
perturbations with shared forcing within each pair.

Parameters, simulator settings,
target construction, and training protocols
appear in Appendix~\ref{app:protocol}.

\paragraph{Evaluation conditions.}
\emph{Seen} uses endpoint-training intervention
values with fresh evaluation examples;
\emph{in-distribution (ID)} uses additional values
inside that range.
\emph{Near-OOD} and \emph{Far-OOD} use values
progressively farther outside it.
These labels describe intervention values.
Response supervision may extend beyond the
endpoint-training range; both domains are
specified in Appendix~\ref{app:protocol}.

\paragraph{Metrics.}
\emph{Endpoint error} compares predicted and
physical mean outcomes.
\emph{Sensitivity error} $E_J$ measures error
in the local response, normalized by the
reference response magnitude.
An exact response gives $E_J=0$; predicting
no response gives $E_J=1$ for a nonzero target.
\emph{Finite-change error} compares
$m_\theta(c,u+\delta)-m_\theta(c,u)$ with the
physical mean change for a nonzero intervention
displacement $\delta$.

Endpoint and finite-change errors use RMSE
for Gaussian and double well, and relative
$\ell_2$ error for PDEBench and SNS.
All errors are lower-is-better.
Table~\ref{tab:id} reports ID means and sample
standard deviations over three training seeds.
Appendix~\ref{app:protocol} defines the metrics
and evaluation budgets;
Appendix~\ref{app:distributional} additionally
assesses outcome variability and distributional fit.

\paragraph{Baselines.}
We compare with unconditional DSBM,
Conditional DSBM, GSBM, and TSBM
\citep{shi2023dsbm,liu2023gsbm,noble2026tsbm}.
These four baselines receive no sensitivity targets;
all except unconditional DSBM receive intervention
inputs. GSBM and TSBM use trajectory costs too.
Our \emph{Sobolev-DSBM (single-rollout)} adaptation
\citep{czarnecki2017sobolev} uses the same
mean-response targets as Tangent-SBM but minimizes
each rollout's squared response error. We evaluate this control on double well and SNS,
where Tangent-SBM uses two-rollout mean supervision.
On Gaussian and PDEBench, Tangent-SBM already uses
single-rollout squared response matching.

\begin{table}[t]
\centering
\small
\setlength{\tabcolsep}{5pt}
\caption{
ID endpoint, sensitivity, and finite-change errors
across four systems. Lower is better.
Entries are mean $\pm$ sample standard deviation
over three training seeds.
Metric definitions and method-specific evaluation
budgets appear in Appendix~\ref{app:protocol}.
}
\label{tab:id}
\begin{tabular}{llccc}
\toprule
Dataset & Method & Endpoint & $E_J$ & Finite change\\
\midrule
Gaussian
& DSBM
& $0.714\pm0.002$ & $1.000\pm0.000$ & $0.321\pm0.000$\\
& Conditional DSBM
& $0.433\pm0.014$ & $0.866\pm0.003$ & $0.275\pm0.001$\\
& GSBM
& $0.440\pm0.002$ & $0.865\pm0.001$ & $0.276\pm0.000$\\
& TSBM
& $0.440\pm0.000$ & $0.865\pm0.001$ & $0.275\pm0.000$\\
& \textbf{Tangent-SBM}
& $\mathbf{0.236\pm0.006}$ & $\mathbf{0.050\pm0.005}$ & $\mathbf{0.016\pm0.001}$\\
\midrule
Double well
& DSBM
& $0.336\pm0.005$ & $1.000\pm0.000$ & $0.110\pm0.000$\\
& Conditional DSBM
& $0.148\pm0.006$ & $0.304\pm0.039$ & $0.033\pm0.005$\\
& GSBM
& $0.151\pm0.008$ & $0.281\pm0.014$ & $0.031\pm0.001$\\
& TSBM
& $0.139\pm0.005$ & $0.288\pm0.015$ & $0.032\pm0.002$\\
& Sobolev-DSBM
& $0.145\pm0.006$ & $0.272\pm0.051$ & $0.030\pm0.006$\\
& \textbf{Tangent-SBM}
& $\mathbf{0.136\pm0.001}$ & $\mathbf{0.197\pm0.031}$ & $\mathbf{0.021\pm0.003}$\\
\midrule
PDEBench
& DSBM
& $0.612\pm0.034$ & $1.000\pm0.000$ & $1.000\pm0.000$\\
& Conditional DSBM
& $0.605\pm0.038$ & $0.847\pm0.054$ & $0.843\pm0.054$\\
& GSBM
& $0.918\pm0.022$ & $0.786\pm0.009$ & $0.784\pm0.009$\\
& TSBM
& $0.904\pm0.013$ & $0.789\pm0.027$ & $0.787\pm0.028$\\
& \textbf{Tangent-SBM}
& $\mathbf{0.598\pm0.069}$ & $\mathbf{0.508\pm0.045}$ & $\mathbf{0.510\pm0.046}$\\
\midrule
SNS
& DSBM
& $0.298\pm0.021$ & $1.000\pm0.000$ & $1.000\pm0.000$\\
& Conditional DSBM
& $0.265\pm0.047$ & $2.658\pm0.359$ & $2.429\pm0.158$\\
& GSBM
& $0.227\pm0.017$ & $3.555\pm0.821$ & $3.207\pm0.768$\\
& TSBM
& $0.176\pm0.002$ & $2.454\pm0.070$ & $2.185\pm0.059$\\
& Sobolev-DSBM
& $0.253\pm0.034$ & $2.624\pm1.103$ & $2.432\pm0.923$\\
& \textbf{Tangent-SBM}
& $\mathbf{0.265\pm0.028}$ & $\mathbf{1.002\pm0.075}$ & $\mathbf{1.291\pm0.354}$\\
\bottomrule
\end{tabular}
\end{table}

\subsection{Endpoint fit and intervention response}
\label{sec:main_results}

On Gaussian ID, Tangent-SBM reduces $E_J$ from
$0.866$ to $0.050$ relative to Conditional DSBM,
while endpoint RMSE decreases from $0.433$ to $0.236$.
GSBM and TSBM have response errors close to the
conditional baseline.
The analytic construction illustrates why fitting
the observed endpoint distributions need not recover
the missing parameter dependence.

On the double well, Tangent-SBM reduces ID response
error from $0.304$ to $0.197$, with endpoint RMSE
decreasing from $0.148$ to $0.136$.
GSBM and TSBM attain response errors of $0.281$
and $0.288$.
The improvement over Conditional DSBM is larger on
Near-OOD ($1.205$ versus $0.303$), while Tangent-SBM's
Far-OOD error remains $0.984$.
On PDEBench, response error decreases by 37--44\%
across evaluated splits relative to Conditional
DSBM, with comparable or lower endpoint errors.
The selected GSBM and TSBM adaptations also have
larger response errors in these runs.

Tangent-SBM reduces SNS ID sensitivity error by
62\% relative to Conditional DSBM and 59\% relative
to TSBM, while improving response alignment from
0.121 and 0.058 to 0.592, respectively
(Appendix~\ref{app:target_coverage},
Table~\ref{tab:response_geometry}).
Its sensitivity error remains close to the
zero-response reference, and its finite-change
error remains above it.
Nevertheless, its improved alignment provides
directional information that a zero response lacks.
Magnitude calibration can rescale these predictions
but cannot improve their alignment.
Section~\ref{sec:decision} evaluates whether this
information supports useful intervention choices
(Appendix~\ref{app:protocol}). Appendix~\ref{app:generalization} reports endpoint,
sensitivity, and finite-change errors across
Seen, Near-OOD, and Far-OOD regions. Tangent-SBM improves intervention responses while
maintaining a similar ability to reproduce the
spread and distribution of outcomes as Conditional
DSBM across the three stochastic systems.
Appendix~\ref{app:distributional} reports these
distributional comparisons.
Figure~\ref{fig:sns} compares endpoint predictions
and intervention-induced field changes for the same
initial condition and viscosity intervention across
methods, with predictions averaged over three
training seeds.

\begin{figure}[t]
    \centering
    \includegraphics[width=0.8\linewidth]
        {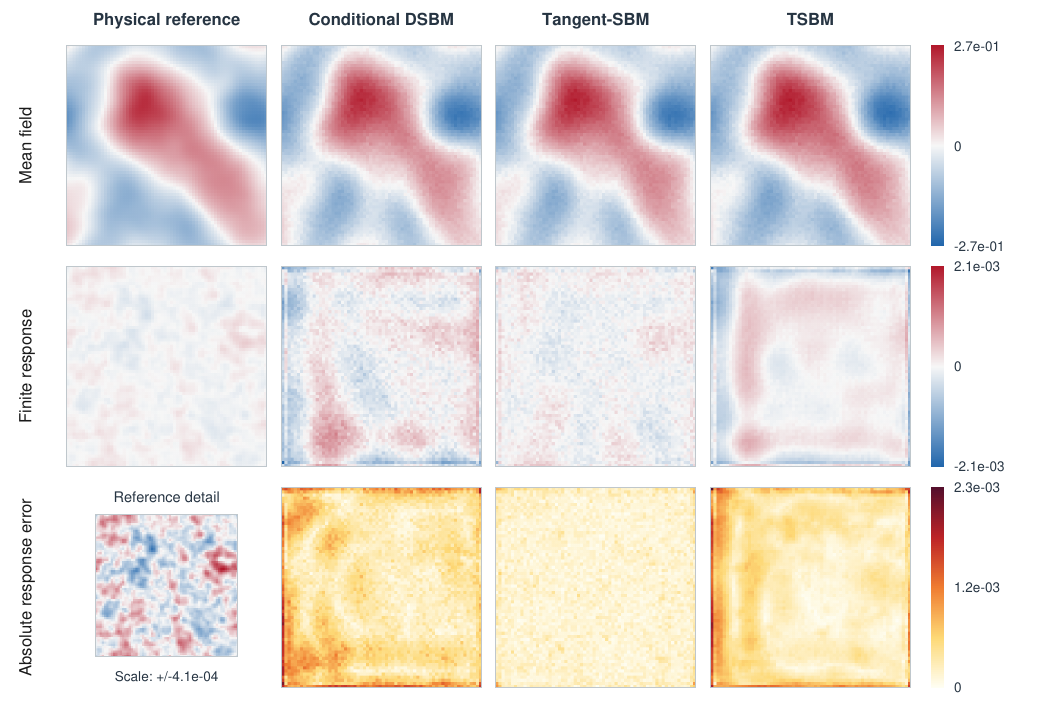}
    \caption{
SNS predictions for a fixed Seen example.
Rows show the mean endpoint field, its change under
a viscosity intervention, and the error in that change. Color scales are shared within rows; the bottom-left
panel repeats the physical change from the second row
on a narrower color scale.
    }
    \label{fig:sns}
\end{figure}

\subsection{What makes response supervision effective?}
\label{sec:ablations}

\paragraph{Matching the meaning of response targets.}
Sobolev-DSBM (single-rollout) matches each
trajectory response to the supplied mean target.
Tangent-SBM instead matches the average response
without the additional variance penalty.
On double well, this reduces ID mean-response
RMSE from $0.104\pm0.025$ to $0.066\pm0.014$.
Appendix~\ref{app:objective_ablations} reports
response variability and compares supervision
through free and endpoint-constrained rollouts.

\paragraph{Correct targets and response coverage.}
On PDEBench ID, assigning sensitivity targets
to the wrong examples or reversing their signs
increases $E_J$ from $0.508$ to $1.010$ and
$1.647$, respectively.
Thus, adding a derivative loss with incorrect
information does not reproduce the gains.

On SNS, retaining 25\%, 50\%, or all additional
response-supervision locations gives ID errors
of $2.075$, $1.726$, and $1.002$.
All variants retain response labels at the
endpoint-training intervention values.
The decreasing mean errors support broader
response coverage in this experiment.
Appendix~\ref{app:target_coverage} reports seed
variability, results across systems, and SNS
direction and magnitude diagnostics, distinguishing
better alignment from merely smaller predictions.

\paragraph{Additional endpoints versus response labels.}
We allocate the response-label simulation budget
to extra endpoint examples instead:
\emph{LocalBudget} samples nearby intervention
values, while \emph{SpreadBudget} samples more
broadly across the intervention domain.
Neither uses derivative targets.

On PDEBench ID, these controls give $E_J$
of approximately $0.85$--$0.88$, compared with
$0.508$ for Tangent-SBM.
On SNS, LocalBudget and SpreadBudget give
$2.929$ and $3.649$, versus $1.002$ for
Tangent-SBM under the same evaluation budget.
The tested endpoint-only allocations therefore
do not recover the gains from explicit response
supervision.
Appendix~\ref{app:budget_controls} specifies
the matched simulator budgets and training procedures.

\paragraph{Response supervision beyond bridges.}
Adding response supervision reduces ID sensitivity
error from $0.865$ to $0.017$ for the Gaussian
terminal predictor and from $0.416$ to $0.204$
for the PDEBench terminal predictor.
Direct response-supervised predictors are strong
alternatives when only mean outcomes are needed;
Tangent-SBM additionally retains a stochastic
model of outcome distributions.
Appendix~\ref{app:alternative_predictors} compares
terminal predictors, deterministic operators,
and conditional flow matching across systems.
Appendix~\ref{app:additional_ablations} examines
forward and backward response supervision
in the Gaussian system.

\subsection{Intervention selection}
\label{sec:decision}

We select a log-viscosity change from
$\{0,\pm0.125,\pm0.25\}$ to approach a desired
SNS mean-field change.
Each method selects the action whose calibrated
local-response prediction is closest to the goal.
The simulator evaluates the chosen action;
the \emph{oracle} selects the best candidate
using simulator-estimated outcomes.

We compare Tangent-SBM with Conditional DSBM,
TSBM (the strongest SNS endpoint baseline),
and unconditional DSBM, whose zero response selects
no intervention.
Testing uses 32 tasks from 16 contexts with two
goals each, shared across three training seeds.
Appendix~\ref{app:protocol} details calibration
and evaluation.


\begin{table}[t]
\centering\small
\caption{Calibrated SNS intervention selection on
32 tasks, shared across three training seeds.
Tracking error and normalized regret are
lower-is-better. Agreement and improvement
percentages average seed-level task rates.}
\label{tab:decisions}
\begin{tabular}{lcccc}
\toprule
Method & Tracking error & Regret & Agreement & Improvement\\
\midrule
DSBM
& $1.000\pm0.000$ & $0.965\pm0.000$
& 0.0\% & 0.0\%\\
Conditional DSBM
& $0.553\pm0.390$ & $0.625\pm0.909$
& 53.1\% & 84.4\%\\
Tangent-SBM
& $0.256\pm0.040$ & $0.048\pm0.028$
& 79.2\% & 100.0\%\\
TSBM
& $0.724\pm0.313$ & $0.742\pm0.460$
& 35.4\% & 47.9\%\\
\midrule
Finite-grid oracle
& $0.186$ & $0$
& 100\% & 100\%\\
\bottomrule
\end{tabular}
\end{table}

Table~\ref{tab:decisions} shows that Tangent-SBM
reduces tracking error, the distance from the
desired change, by 54\% relative to Conditional
DSBM and 65\% relative to TSBM.
It improves over no intervention on all 96
seed--task pairs and selects the simulator's
best candidate action in 79.2\% of decisions.
Its normalized regret, measuring additional
squared tracking loss relative to that oracle,
is $0.048$.
These results demonstrate useful intervention
selection.

\section{Discussion and limitations}

\method\ shows how mechanistic sensitivities can
complement endpoint observations to improve
intervention responses in stochastic transport.
Across four systems, the comparisons with
conditional bridges and simulator-budget controls
support the value of explicit response supervision.
On the stochastic benchmarks, these gains coexist
with distributional fit comparable to Conditional
DSBM. The SNS decision study further
shows that learned responses can guide useful
intervention choices.

Better endpoint predictions do not necessarily yield
better intervention decisions.
Benefits depend on sensitivity quality and coverage,
and tangent supervision adds computation.
Tangent-SBM offers a practical way to learn intervention
responses alongside outcome distributions.


\newpage

\subsection*{AI use statement}

We have not used generative AI tools for any of the tasks with required
disclosure: generating synthetic data sets, developing theoretical models or
conceptual frameworks, formulating mathematical claims, providing ingredients
for or writing proofs, proposing or refining hypotheses, designing or giving
feedback on research methodology or experiments, implementing methods, cleaning
or reformatting datasets, or interpreting results. Translation and qualitative
or thematic data analysis are not applicable to this work. We used generative AI
tools only to edit author-written text for readability, including grammar,
phrasing, and concision. All technical content, including the method,
derivations, experiments, and analysis, was developed and written by the
authors. Every AI-suggested edit was reviewed by the authors, who accepted,
modified, or rejected it and checked that it did not change the technical
meaning of the original text. We take responsibility for the final content of
this work.

\subsection*{Ethics statement}

This work studies simulated physical systems and
does not involve human participants or personal data.
Its intended application is scientific prediction
and intervention analysis.

\subsection*{Reproducibility statement}
Section~\ref{sec:method} specifies the model,
tangent propagation, and response-supervision
objectives.
Appendix~\ref{app:proofs} provides assumptions
and proofs for the theoretical statements.
Appendix~\ref{app:protocol} documents dataset
construction, model configurations, training
procedures, evaluation metrics, sampling budgets,
and intervention-selection protocols.
Appendices~\ref{app:objective_ablations},
\ref{app:target_coverage}, and
\ref{app:budget_controls} describe the objective,
target, and simulator-budget controls.
Results are reported across three training seeds,
with variability summarized using sample
standard deviations.
Source code, evaluation scripts, and setup
instructions are available at
\url{https://anonymous.4open.science/r/Tangent-SBM-4508/README.md}.





\newpage

\bibliography{references}

\begin{thebibliography}{36}
\providecommand{\natexlab}[1]{#1}
\providecommand{\url}[1]{\texttt{#1}}
\expandafter\ifx\csname urlstyle\endcsname\relax
  \providecommand{\doi}[1]{doi: #1}\else
  \providecommand{\doi}{doi: \begingroup \urlstyle{rm}\Url}\fi

\bibitem[Albergo et~al.(2023)Albergo, Boffi, and Vanden-Eijnden]{albergo2023interpolants}
Michael~S. Albergo, Nicholas~M. Boffi, and Eric Vanden-Eijnden.
\newblock Stochastic interpolants: A unifying framework for flows and diffusions.
\newblock \emph{arXiv preprint arXiv:2303.08797}, 2023.
\newblock URL \url{https://arxiv.org/abs/2303.08797}.

\bibitem[Baird(1995)]{baird1995residual}
Leemon Baird.
\newblock Residual algorithms: Reinforcement learning with function approximation.
\newblock In \emph{Proceedings of the Twelfth International Conference on Machine Learning}, pp.\  30--37, 1995.
\newblock \doi{10.1016/B978-1-55860-377-6.50013-X}.

\bibitem[Baydin et~al.(2018)Baydin, Pearlmutter, Radul, and Siskind]{baydin2018autodiff}
Atilim~Gunes Baydin, Barak~A. Pearlmutter, Alexey~Andreyevich Radul, and Jeffrey~Mark Siskind.
\newblock Automatic differentiation in machine learning: A survey.
\newblock \emph{Journal of Machine Learning Research}, 18\penalty0 (153):\penalty0 1--43, 2018.
\newblock URL \url{https://jmlr.org/papers/v18/17-468.html}.

\bibitem[Bonnans \& Shapiro(2000)Bonnans and Shapiro]{bonnans2000perturbation}
J.~Fr{\'e}d{\'e}ric Bonnans and Alexander Shapiro.
\newblock \emph{Perturbation Analysis of Optimization Problems}.
\newblock Springer, 2000.
\newblock \doi{10.1007/978-1-4612-1394-9}.

\bibitem[Cuturi(2013)]{cuturi2013sinkhorn}
Marco Cuturi.
\newblock Sinkhorn distances: Lightspeed computation of optimal transport.
\newblock In \emph{Advances in Neural Information Processing Systems}, 2013.
\newblock URL \url{http://papers.neurips.cc/paper/4927-sinkhorn-distances-lightspeed-computation-of-optimal-transport.pdf}.

\bibitem[Czarnecki et~al.(2017)Czarnecki, Osindero, Jaderberg, Swirszcz, and Pascanu]{czarnecki2017sobolev}
Wojciech~Marian Czarnecki, Simon Osindero, Max Jaderberg, Grzegorz Swirszcz, and Razvan Pascanu.
\newblock Sobolev training for neural networks.
\newblock In \emph{Advances in Neural Information Processing Systems}, 2017.
\newblock URL \url{https://arxiv.org/abs/1706.04859}.

\bibitem[De~Bortoli et~al.(2021)De~Bortoli, Thornton, Heng, and Doucet]{debortoli2021diffusion}
Valentin De~Bortoli, James Thornton, Jeremy Heng, and Arnaud Doucet.
\newblock Diffusion {Schr{\"o}dinger} bridge with applications to score-based generative modeling.
\newblock In \emph{Advances in Neural Information Processing Systems}, 2021.
\newblock URL \url{https://arxiv.org/abs/2106.01357}.

\bibitem[Glasserman(1991)]{glasserman1991gradient}
Paul Glasserman.
\newblock \emph{Gradient Estimation via Perturbation Analysis}.
\newblock Springer, 1991.
\newblock URL \url{https://link.springer.com/book/9780792390954}.

\bibitem[Ho et~al.(2020)Ho, Jain, and Abbeel]{ho2020denoising}
Jonathan Ho, Ajay Jain, and Pieter Abbeel.
\newblock Denoising diffusion probabilistic models.
\newblock In \emph{Advances in Neural Information Processing Systems}, 2020.
\newblock URL \url{https://arxiv.org/abs/2006.11239}.

\bibitem[Karniadakis et~al.(2021)Karniadakis, Kevrekidis, Lu, Perdikaris, Wang, and Yang]{karniadakis2021physics}
George~Em Karniadakis, Yannis Kevrekidis, Lu~Lu, Paris Perdikaris, Sifan Wang, and Liu Yang.
\newblock Physics-informed machine learning.
\newblock \emph{Nature Reviews Physics}, 3:\penalty0 422--440, 2021.
\newblock \doi{10.1038/s42254-021-00314-5}.

\bibitem[Kidger et~al.(2021)Kidger, Foster, Li, and Lyons]{kidger2021sde}
Patrick Kidger, James Foster, Xuechen Li, and Terry~J. Lyons.
\newblock Neural {SDEs} as infinite-dimensional {GANs}.
\newblock In \emph{Proceedings of the 38th International Conference on Machine Learning}, volume 139 of \emph{Proceedings of Machine Learning Research}, pp.\  5453--5463, 2021.
\newblock URL \url{https://proceedings.mlr.press/v139/kidger21b.html}.

\bibitem[Kovachki et~al.(2023)Kovachki, Li, Liu, Azizzadenesheli, Bhattacharya, Stuart, and Anandkumar]{kovachki2023operator}
Nikola Kovachki, Zongyi Li, Burigede Liu, Kamyar Azizzadenesheli, Kaushik Bhattacharya, Andrew Stuart, and Anima Anandkumar.
\newblock Neural operator: Learning maps between function spaces with applications to {PDEs}.
\newblock \emph{Journal of Machine Learning Research}, 24\penalty0 (89):\penalty0 1--97, 2023.
\newblock URL \url{https://jmlr.org/papers/v24/21-1524.html}.

\bibitem[L{\'e}onard(2014)]{leonard2014survey}
Christian L{\'e}onard.
\newblock A survey of the {Schr\"odinger} problem and some of its connections with optimal transport.
\newblock \emph{Discrete and Continuous Dynamical Systems}, 34\penalty0 (4):\penalty0 1533--1574, 2014.
\newblock \doi{10.3934/dcds.2014.34.1533}.

\bibitem[Li et~al.(2020)Li, Wong, Chen, and Duvenaud]{li2020scalable}
Xuechen Li, Ting-Kam~Leonard Wong, Ricky T.~Q. Chen, and David Duvenaud.
\newblock Scalable gradients for stochastic differential equations.
\newblock In \emph{Proceedings of the International Conference on Artificial Intelligence and Statistics}, volume 108 of \emph{Proceedings of Machine Learning Research}, pp.\  3870--3882, 2020.
\newblock URL \url{https://proceedings.mlr.press/v108/li20i.html}.

\bibitem[Li et~al.(2025)Li, Dou, Fang, Han, Deng, and Yang]{li2025physicsaligned}
Zeyu Li, Hongkun Dou, Shen Fang, Wang Han, Yue Deng, and Lijun Yang.
\newblock Physics-aligned field reconstruction with diffusion bridge.
\newblock In \emph{International Conference on Learning Representations}, 2025.
\newblock URL \url{https://openreview.net/forum?id=D042vFwJAM}.

\bibitem[Li et~al.(2021)Li, Kovachki, Azizzadenesheli, Liu, Bhattacharya, Stuart, and Anandkumar]{li2021fourier}
Zongyi Li, Nikola Kovachki, Kamyar Azizzadenesheli, Burigede Liu, Kaushik Bhattacharya, Andrew Stuart, and Anima Anandkumar.
\newblock Fourier neural operator for parametric partial differential equations.
\newblock In \emph{International Conference on Learning Representations}, 2021.
\newblock URL \url{https://openreview.net/forum?id=c8P9NQVtmnO}.

\bibitem[Lipman et~al.(2023)Lipman, Chen, Ben-Hamu, Nickel, and Le]{lipman2023flow}
Yaron Lipman, Ricky T.~Q. Chen, Heli Ben-Hamu, Maximilian Nickel, and Matthew Le.
\newblock Flow matching for generative modeling.
\newblock In \emph{International Conference on Learning Representations}, 2023.
\newblock URL \url{https://arxiv.org/abs/2210.02747}.

\bibitem[Liu et~al.(2022)Liu, Chen, So, and Theodorou]{liu2022deepgsb}
Guan-Horng Liu, Tianrong Chen, Oswin So, and Evangelos~A. Theodorou.
\newblock Deep generalized {Schr{\"o}dinger} bridge.
\newblock In \emph{Advances in Neural Information Processing Systems}, 2022.
\newblock URL \url{https://arxiv.org/abs/2209.09893}.

\bibitem[Liu et~al.(2024)Liu, Lipman, Nickel, Karrer, Theodorou, and Chen]{liu2023gsbm}
Guan-Horng Liu, Yaron Lipman, Maximilian Nickel, Brian Karrer, Evangelos~A. Theodorou, and Ricky T.~Q. Chen.
\newblock Generalized {Schr\"odinger} bridge matching.
\newblock In \emph{International Conference on Learning Representations}, 2024.
\newblock URL \url{https://arxiv.org/abs/2310.02233}.

\bibitem[Lu et~al.(2021)Lu, Jin, Pang, Zhang, and Karniadakis]{lu2021deeponet}
Lu~Lu, Pengzhan Jin, Guofei Pang, Zhongqiang Zhang, and George~Em Karniadakis.
\newblock Learning nonlinear operators via {DeepONet} based on the universal approximation theorem of operators.
\newblock \emph{Nature Machine Intelligence}, 3:\penalty0 218--229, 2021.
\newblock \doi{10.1038/s42256-021-00302-5}.

\bibitem[Luo et~al.(2025)Luo, O'Leary-Roseberry, Chen, and Ghattas]{luo2025optimization}
Dingcheng Luo, Thomas O'Leary-Roseberry, Peng Chen, and Omar Ghattas.
\newblock Efficient {PDE}-constrained optimization under high-dimensional uncertainty using derivative-informed neural operators.
\newblock \emph{SIAM Journal on Scientific Computing}, 2025.
\newblock \doi{10.1137/23M157956X}.

\bibitem[Noble et~al.(2026)Noble, Scheid, Janati, Moulines, and Durmus]{noble2026tsbm}
Maxence Noble, Marie Scheid, Yazid Janati, Eric Moulines, and Alain Durmus.
\newblock Twisted {Schr\"odinger} bridge matching.
\newblock \emph{arXiv preprint arXiv:2607.16987}, 2026.
\newblock URL \url{https://arxiv.org/abs/2607.16987}.

\bibitem[O'Leary-Roseberry et~al.(2024)O'Leary-Roseberry, Chen, Villa, and Ghattas]{oleary2024dino}
Thomas O'Leary-Roseberry, Peng Chen, Umberto Villa, and Omar Ghattas.
\newblock Derivative-informed neural operator: An efficient framework for high-dimensional parametric derivative learning.
\newblock \emph{Journal of Computational Physics}, 496:\penalty0 112555, 2024.
\newblock \doi{10.1016/j.jcp.2023.112555}.

\bibitem[Pavliotis(2014)]{pavliotis2014stochastic}
Grigorios~A. Pavliotis.
\newblock \emph{Stochastic Processes and Applications: Diffusion Processes, the Fokker--Planck and Langevin Equations}.
\newblock Springer, 2014.
\newblock \doi{10.1007/978-1-4939-1323-7}.

\bibitem[Pearlmutter(1994)]{pearlmutter1994fast}
Barak~A. Pearlmutter.
\newblock Fast exact multiplication by the hessian.
\newblock \emph{Neural Computation}, 6\penalty0 (1):\penalty0 147--160, 1994.
\newblock \doi{10.1162/neco.1994.6.1.147}.

\bibitem[Peyr{\'e} \& Cuturi(2019)Peyr{\'e} and Cuturi]{peyre2019computational}
Gabriel Peyr{\'e} and Marco Cuturi.
\newblock Computational optimal transport.
\newblock \emph{Foundations and Trends in Machine Learning}, 11\penalty0 (5--6):\penalty0 355--607, 2019.
\newblock \doi{10.1561/2200000073}.

\bibitem[Raissi et~al.(2019)Raissi, Perdikaris, and Karniadakis]{raissi2019physics}
Maziar Raissi, Paris Perdikaris, and George~E Karniadakis.
\newblock Physics-informed neural networks: A deep learning framework for solving forward and inverse problems involving nonlinear partial differential equations.
\newblock \emph{Journal of Computational Physics}, 378:\penalty0 686--707, 2019.
\newblock \doi{10.1016/j.jcp.2018.10.045}.

\bibitem[Shi et~al.(2022)Shi, De~Bortoli, Deligiannidis, and Doucet]{shi2022conditional}
Yuyang Shi, Valentin De~Bortoli, George Deligiannidis, and Arnaud Doucet.
\newblock Conditional simulation using diffusion {Schr{\"o}dinger} bridges.
\newblock In \emph{Proceedings of the Thirty-Eighth Conference on Uncertainty in Artificial Intelligence}, volume 180 of \emph{Proceedings of Machine Learning Research}, pp.\  1792--1802. PMLR, 2022.
\newblock URL \url{https://proceedings.mlr.press/v180/shi22a.html}.

\bibitem[Shi et~al.(2023)Shi, De~Bortoli, Campbell, and Doucet]{shi2023dsbm}
Yuyang Shi, Valentin De~Bortoli, Andrew Campbell, and Arnaud Doucet.
\newblock Diffusion {Schr\"odinger} bridge matching.
\newblock In \emph{Advances in Neural Information Processing Systems}, 2023.
\newblock URL \url{https://arxiv.org/abs/2303.16852}.

\bibitem[Song et~al.(2021)Song, Sohl-Dickstein, Kingma, Kumar, Ermon, and Poole]{song2021score}
Yang Song, Jascha Sohl-Dickstein, Diederik~P. Kingma, Abhishek Kumar, Stefano Ermon, and Ben Poole.
\newblock Score-based generative modeling through stochastic differential equations.
\newblock In \emph{International Conference on Learning Representations}, 2021.
\newblock URL \url{https://arxiv.org/abs/2011.13456}.

\bibitem[Srinivas \& Fleuret(2018)Srinivas and Fleuret]{srinivas2018jacobian}
Suraj Srinivas and Francois Fleuret.
\newblock Knowledge transfer with jacobian matching.
\newblock In \emph{Proceedings of the International Conference on Machine Learning}, volume~80 of \emph{Proceedings of Machine Learning Research}, 2018.
\newblock URL \url{https://proceedings.mlr.press/v80/srinivas18a.html}.

\bibitem[Takamoto et~al.(2022)Takamoto, Praditia, Leiteritz, MacKinlay, Alesiani, Pfl{\"u}ger, and Niepert]{takamoto2022pdebench}
Makoto Takamoto, Timothy Praditia, Raphael Leiteritz, Dan MacKinlay, Francesco Alesiani, Dirk Pfl{\"u}ger, and Mathias Niepert.
\newblock {PDEBench}: An extensive benchmark for scientific machine learning.
\newblock In \emph{Advances in Neural Information Processing Systems, Datasets and Benchmarks Track}, 2022.
\newblock URL \url{https://arxiv.org/abs/2210.07182}.

\bibitem[Tong et~al.(2024)Tong, Malkin, Fatras, Atanackovic, Zhang, Huguet, Wolf, and Bengio]{tong2024simulationfree}
Alexander Tong, Nikolay Malkin, Kilian Fatras, Lazar Atanackovic, Yanlei Zhang, Guillaume Huguet, Guy Wolf, and Yoshua Bengio.
\newblock Simulation-free {Schr{\"o}dinger} bridges via score and flow matching.
\newblock In \emph{Proceedings of the 27th International Conference on Artificial Intelligence and Statistics}, volume 238 of \emph{Proceedings of Machine Learning Research}, pp.\  1279--1287. PMLR, 2024.
\newblock URL \url{https://proceedings.mlr.press/v238/tong24a.html}.

\bibitem[Wang et~al.(2021)Wang, Jiao, Xu, Wang, and Yang]{wang2021deep}
Gefei Wang, Yuling Jiao, Qian Xu, Yang Wang, and Can Yang.
\newblock Deep generative learning via {Schr{\"o}dinger} bridge.
\newblock In \emph{Proceedings of the International Conference on Machine Learning}, volume 139 of \emph{Proceedings of Machine Learning Research}, pp.\  10794--10804, 2021.
\newblock URL \url{https://proceedings.mlr.press/v139/wang21l.html}.

\bibitem[Yao et~al.(2026)Yao, Luo, Cao, Kovachki, O'Leary-Roseberry, and Ghattas]{yao2025difno}
Boyuan Yao, Dingcheng Luo, Lianghao Cao, Nikola Kovachki, Thomas O'Leary-Roseberry, and Omar Ghattas.
\newblock Derivative-informed fourier neural operator: Universal approximation and applications to {PDE}-constrained optimization.
\newblock \emph{arXiv preprint arXiv:2512.14086}, 2026.
\newblock URL \url{https://arxiv.org/abs/2512.14086}.

\bibitem[Zhu et~al.(2025)Zhu, Li, Shi, Thompson, Nash, Lara~Rangel, Li, Chen, Zhu, Meng, and Ni]{zhu2025spdebench}
Yuantu Zhu, Zheyan Li, Dai Shi, Luke Thompson, Oliver Nash, Jose~Miguel Lara~Rangel, Siran Li, Bingguang Chen, Rongchan Zhu, Qi~Meng, and Hao Ni.
\newblock {SPDEBench}: An extensive benchmark for learning stochastic {PDEs}.
\newblock \emph{arXiv preprint arXiv:2505.18511}, 2025.
\newblock URL \url{https://arxiv.org/abs/2505.18511}.

\end{thebibliography}
\bibliographystyle{iclr2027_conference}
\clearpage
\appendix
\section{Proofs and additional analysis}
\label{app:proofs}

We give proofs of the main-text propositions and supporting
results for the response objectives.
Throughout, vector norms are Euclidean, matrix operator
norms are induced by the Euclidean norm, and the context
$c$ is fixed unless stated otherwise.

\subsection{Trajectory sensitivities and mean responses}
\label{app:regularity}

The methods section uses two differentiation steps:
differentiating a stochastic rollout at fixed randomness,
and interchanging differentiation with expectation.

\paragraph{Pathwise differentiation.}
Sufficient conditions include a drift
$b_\theta(x,u,t)$ that is continuously differentiable
in $(x,u)$, with uniformly bounded first derivatives
on the intervention neighborhood under consideration,
together with standard growth and time-regularity
conditions ensuring a unique solution.
Assume that the additive diffusion is independent of
$(x,u)$ and square-integrable in time, and that the
initial state is independent of $u$.

For a fixed Brownian path, the noise term does not change
when $u$ changes.
Differentiating the integral form of the state equation
therefore gives
\[
R_t^v
=
\int_0^t
\left[
\partial_x b_\theta(X_s,u,s)R_s^v
+
\partial_u b_\theta(X_s,u,s)v
\right]ds.
\]
This is the pathwise sensitivity
\eqref{eq:tangent}
\citep{glasserman1991gradient,li2020scalable}.

\paragraph{Interchanging derivative and expectation.}
Suppose
\[
\|\partial_x b_\theta\|_{\mathrm{op}}\leq L,
\qquad
\|\partial_u b_\theta\|_{\mathrm{op}}\leq M
\]
uniformly in state, time, and the intervention neighborhood.
For a unit $v$, Gr\"onwall's inequality yields
\begin{equation}
\|R_T^v\|_2\leq M\phi_L(T),
\qquad
\phi_L(T)=
\begin{cases}
(e^{LT}-1)/L, & L>0,\\
T, & L=0.
\end{cases}
\label{eq:app_tangent_domination}
\end{equation}
For the identity observable $h(X_T,u)=X_T$, this is
a deterministic bound on the output derivative.
Assuming the output is integrable, it permits
differentiation under expectation.

For a general observable, a sufficient condition is
almost-sure continuous differentiability of
$Y_\theta(c,u;W)$ in $u$ and an integrable random
variable $H(W)$ satisfying
\[
\sup_{|\epsilon|<\rho}
\left\|
\partial_u Y_\theta(c,u+\epsilon v;W)v
\right\|_2
\leq H(W)
\]
for some $\rho>0$.
Dominated convergence then gives
\begin{equation}
\partial_u m_\theta(c,u)v
=
\mathbb{E}
\left[
\partial_u Y_\theta(c,u;W)v
\mid c,u,v
\right]
=
\mathbb{E}[Z_\theta\mid c,u,v].
\label{eq:app_mean_response}
\end{equation}

For a differentiable Euler--Maruyama rollout, the discrete
tangent obeys the analogous recursion
\[
R_{n+1}
=
R_n+\Delta t\,
\left[
\partial_x b_\theta(X_n,u,t_n)R_n
+
\partial_u b_\theta(X_n,u,t_n)v
\right].
\]
The same bounds imply
$\|R_N\|_2\leq MT e^{LT}$ for $T=N\Delta t$.
This controls the derivative of the discrete computation;
identifying it with the continuous-SDE derivative
additionally requires numerical convergence.

The conditions above concern differentiation with respect
to the intervention.
Differentiating the training objective with respect to
$\theta$ requires the separate integrability conditions
stated below.

\subsection{Endpoint ambiguity}

\begin{proof}[Proof of Proposition~\ref{prop:nonid}]
Choose a vector $d\in\mathbb{R}^p$ such that the scalar
projections $d^\top u_i$ are all distinct.
Such a vector exists because the excluded choices lie
in finitely many proper hyperplanes
$d^\top(u_i-u_j)=0$.

Let
\[
q(u)=\prod_{i=1}^k(d^\top u-d^\top u_i),
\]
choose a nonzero output vector $w$, and let $Y(u)$
have a smooth conditional density and finite mean.
For example, $Y(u)$ may be a Gaussian random vector
whose distribution is independent of $u$.
Define
\[
\widetilde Y(u)=Y(u)+a q(u)w.
\]
At every $u_i$, $q(u_i)=0$, so the two outcomes have
identical conditional distributions.

Their mean Jacobians differ by
\[
a w\nabla q(u)^\top.
\]
At an observed intervention,
\[
\nabla q(u_i)
=
d\prod_{j\neq i}d^\top(u_i-u_j)\neq 0.
\]
Consequently, the Jacobian discrepancy at any chosen
$u_i$ becomes arbitrarily large as $|a|$ increases.
Translation preserves smoothness of the conditional
density and finiteness of the mean.
\end{proof}

This construction can also be realized by additive-noise
stochastic transports.
Start two processes from the same intervention-independent
state with identical noise.
Give one zero drift and the other the constant-in-time
drift $a q(u)w/T$.
Their terminal states then differ by exactly $a q(u)w$.

The construction does not impose a common upper bound on
derivatives across all choices of $a$.
Such additional restrictions can limit the ambiguity.
Moreover, replacing $q$ by $q^2$ preserves both the
endpoint distributions and their mean Jacobians at the
observed interventions, while changing behavior elsewhere.
Finite endpoint and sensitivity observations therefore
do not identify the entire response function.

\subsection{Response objectives and independent rollouts}
\label{app:meanloss}

All expectations in this subsection condition on a fixed
$(c,u,v)$.

\paragraph{Realization-level pairing.}
Fix $(c,u,v)$ and let $\pi$ be a $\theta$-independent
coupling of the random inputs $(W,\xi)$.
The responses $Z_\theta(W)$ and $Z^*(\xi)$ induce a joint
response distribution that may depend on $\theta$.
All expectations below are taken over the fixed
random-input coupling $\pi$.
Set
\[
\bar z_\theta=\mathbb{E}_\pi[Z_\theta],
\qquad
\bar z^*=\mathbb{E}_\pi[Z^*].
\]
Jensen's inequality gives
\[
\|\bar z_\theta-\bar z^*\|_2^2
=
\|\mathbb{E}_\pi[Z_\theta-Z^*]\|_2^2
\leq
\mathbb{E}_\pi\|Z_\theta-Z^*\|_2^2.
\]
This proves \eqref{eq:path_mean_bound}.

The full objective depends on the pairing.
Writing
$\Sigma_\theta=\operatorname{Cov}(Z_\theta)$,
$\Sigma_*=\operatorname{Cov}(Z^*)$, and
\[
C_\pi=
\mathbb{E}_\pi
\left[
(Z_\theta-\bar z_\theta)(Z^*-\bar z^*)^\top
\right],
\]
gives
\begin{align}
\mathbb{E}_\pi\|Z_\theta-Z^*\|_2^2
&=
\|\bar z_\theta-\bar z^*\|_2^2
+
\operatorname{tr}(\Sigma_\theta)
+
\operatorname{tr}(\Sigma_*)
-
2\operatorname{tr}(C_\pi).
\label{eq:app_pairing}
\end{align}
Under independent pairing, $C_\pi=0$.
The resulting loss therefore penalizes model response
variance as well as mean error.
A realization-level interpretation requires the pairing
to be specified; it cannot be inferred from endpoint
marginals alone.

\begin{proposition}[Conditional-mean response objective]
\label{prop:cross}
Let $Z_1$ and $Z_2$ be independent model response samples
with finite second moments and common mean
$\bar z_\theta$.
For a deterministic target $\bar z^*$,
\begin{align}
\mathbb{E}
\left[
\langle Z_1-\bar z^*,Z_2-\bar z^*\rangle
\right]
&=
\|\bar z_\theta-\bar z^*\|_2^2,
\label{eq:app_cross_identity}
\\
\mathbb{E}\|Z_1-\bar z^*\|_2^2
&=
\|\bar z_\theta-\bar z^*\|_2^2
+
\mathbb{E}\|Z_1-\bar z_\theta\|_2^2.
\label{eq:app_variance_identity}
\end{align}
If differentiation with respect to $\theta$ can be
interchanged with expectation, differentiating both
response factors in the cross-product loss gives an
unbiased gradient of
$\|\bar z_\theta-\bar z^*\|_2^2$.
\end{proposition}

\begin{proof}
Write $Z_i=\bar z_\theta+\xi_i$, with
$\mathbb{E}[\xi_i]=0$.
Independence gives
$\mathbb{E}\langle\xi_1,\xi_2\rangle=0$.
Expanding the cross product leaves
$\|\bar z_\theta-\bar z^*\|_2^2$.
Expanding the single-response squared loss instead
leaves the additional term
$\mathbb{E}\|\xi_1\|_2^2$.

For the gradient identity, assume that the sampled loss
is almost surely differentiable in $\theta$ and its
gradient has an integrable envelope in a neighborhood
of the parameter value.
Assume also that the random-input distribution and
target are independent of $\theta$.
Differentiating under expectation yields
\[
\mathbb{E}
\left[
\nabla_\theta
\langle Z_1-\bar z^*,Z_2-\bar z^*\rangle
\right]
=
\nabla_\theta
\|\bar z_\theta-\bar z^*\|_2^2.
\]
The derivative on the left includes both factors.
\end{proof}

This is an application of the independent-sampling
principle for squared expected residuals
\citep{baird1995residual}.
It removes the additional response-variance penalty
without specifying the correct response variance or
guaranteeing low estimator variance.

\subsection{Biased and noisy mechanistic targets}
\label{app:noisytarget}

Fix $(c,u,v)$.
Let a simulator-generated target $\widehat z^*$ satisfy
\[
\mathbb{E}[\widehat z^*]=\bar z^*+b,
\qquad
\operatorname{Cov}(\widehat z^*)=\Sigma_*,
\]
where $b$ is its bias relative to the physical
mean derivative.
Assume that the target is independent of the two
independent model response samples.

When the same target estimate is reused in both
residuals,
\begin{equation}
\mathbb{E}
\left[
\langle Z_1-\widehat z^*,Z_2-\widehat z^*\rangle
\right]
=
\|\bar z_\theta-\bar z^*-b\|_2^2
+
\operatorname{tr}(\Sigma_*).
\label{eq:app_noisy_target}
\end{equation}
To prove this, condition on $\widehat z^*$.
Independence of the model responses gives
$\|\bar z_\theta-\widehat z^*\|_2^2$.
Taking expectation over the target and applying the
variance decomposition gives
\eqref{eq:app_noisy_target}.

For a fixed target-generation mechanism, the covariance
term is independent of $\theta$ and does not change the
population gradient.
Bias does change the fitted response: the population
objective targets $\bar z^*+b$.

Using independent target estimates in the two residuals,
also independent of the model responses, removes the
covariance term.
It does not remove bias in those estimates.
A frozen target bank instead supervises its realized
values, whose discrepancy from $\bar z^*$ contributes
to $\varepsilon_{\mathrm{target}}$.

For example, a centered directional finite difference
\[
\frac{
m^*(u+h v)-m^*(u-h v)
}{2h}
\]
approximates $J^*(u)v$ with $O(h^2)$ truncation error
when the third directional derivative is bounded.
Finite Monte Carlo estimation and numerical simulation
can contribute additional error.
These errors belong to the target-approximation term;
they are not all represented by target covariance.

\subsection{Finite-intervention error}

\begin{proof}[Proof of Proposition~\ref{thm:finite}]
Set $e(u)=m_\theta(u)-m^*(u)$.
The fundamental theorem of calculus gives
\begin{align*}
D_\theta-D^*
&=
e(u+\delta)-e(u)\\
&=
r\int_0^1
[J_\theta(u+s\delta)-J^*(u+s\delta)]v\,ds.
\end{align*}
Applying Cauchy--Schwarz on the unit interval yields
\[
\|D_\theta-D^*\|_2^2
\leq
r^2\int_0^1
\|[J_\theta(u+s\delta)-J^*(u+s\delta)]v\|_2^2\,ds.
\]
\end{proof}

To obtain \eqref{eq:targetbound}, write
\begin{align*}
[J_\theta(u+s\delta)-J^*(u+s\delta)]v
&=
[J_\theta(u+s\delta)v-\widetilde z(s)]\\
&\quad+
[\widetilde z(s)-J^*(u+s\delta)v].
\end{align*}
The triangle inequality in
$L^2([0,1];\mathbb{R}^d)$ bounds the segment error by
$\varepsilon_{\mathrm{fit}}
+\varepsilon_{\mathrm{target}}$.
Taking square roots in \eqref{eq:finitebound}
then proves \eqref{eq:targetbound}.

\subsection{Direction and collocation coverage}
\label{app:coverage}

\paragraph{Random directions.}
For a fixed matrix $A$ and a random unit direction
satisfying $\mathbb{E}[vv^\top]=I/p$,
\[
\mathbb{E}_v\|Av\|_2^2
=
\operatorname{tr}
\left(A^\top A\,\mathbb{E}[vv^\top]\right)
=
\frac{\|A\|_F^2}{p}.
\]
For any unit $w$,
\[
\|Aw\|_2^2
\leq\|A\|_{\mathrm{op}}^2
\leq\|A\|_F^2
=
p\,\mathbb{E}_v\|Av\|_2^2.
\]
Taking $A=J_\theta(u)-J^*(u)$ shows that the
expected squared sensitivity error over isotropic
directions bounds the error in any unit direction
by a factor of $p$.

\paragraph{Finite collocation points.}
Let
\[
S=\{u+s\delta:s\in[0,1]\},
\qquad
E_v(x)=[J_\theta(x)-J^*(x)]v.
\]
Suppose a finite set $\mathcal C\subset S$ has fill
distance
\[
h=\sup_{x\in S}\min_{z\in\mathcal C}\|x-z\|_2,
\]
and assume
\[
\max_{z\in\mathcal C}\|E_v(z)\|_2\leq\varepsilon,
\qquad
\|E_v(x)-E_v(z)\|_2\leq L\|x-z\|_2
\quad (x,z\in S).
\]
For every $x\in S$, choose a nearest collocation
point $z$.
Then
\[
\|E_v(x)\|_2
\leq
\|E_v(z)\|_2+L\|x-z\|_2
\leq
\varepsilon+Lh.
\]
Proposition~\ref{thm:finite} therefore gives
\begin{equation}
\|D_\theta-D^*\|_2
\leq r(\varepsilon+Lh).
\label{eq:app_collocation_bound}
\end{equation}

The bound assumes a maximum response error at the
collocation points and a Lipschitz bound between them.

\paragraph{Accuracy at the base intervention.}
If only $\|E_v(u)\|_2\leq\varepsilon_0$ is available
and $E_v$ is $L$-Lipschitz on $S$, then
\[
\|E_v(u+s\delta)\|_2
\leq\varepsilon_0+Lrs.
\]
Integrating gives
\begin{equation}
\|D_\theta-D^*\|_2
\leq r\varepsilon_0+\frac{Lr^2}{2}.
\label{eq:app_local_bound}
\end{equation}

\subsection{Intervention selection and calibrated predictions}
\label{app:decision}

\begin{proof}[Proof of Proposition~\ref{cor:decision}]
The prediction-error bound, optimality of $\widehat a$,
and the triangle inequality give
\begin{align*}
\|g(\widehat a)-y\|_2
&\leq
\|\widehat g(\widehat a)-y\|_2+\eta\\
&\leq
\|\widehat g(a^*)-y\|_2+\eta\\
&\leq
\|g(a^*)-y\|_2+2\eta.
\end{align*}
Subtracting the oracle residual proves
\eqref{eq:tracking_regret}.
Squaring the last inequality and subtracting
$\|g(a^*)-y\|_2^2$ gives
\[
\ell(\widehat a)-\ell(a^*)
\leq
4\eta\|g(a^*)-y\|_2+4\eta^2,
\]
which proves \eqref{eq:regret}.
\end{proof}

Minimizing the residual norm or its square gives the
same set of minimizers.
The proposition allows ties and does not require
agreement with a particular oracle action.

\paragraph{Finite-change predictions.}
For an action displacement $\delta_a=r_a v_a$, set
\[
g(a)=m^*(u+\delta_a)-m^*(u),
\qquad
\widehat g(a)=m_\theta(u+\delta_a)-m_\theta(u).
\]
Applying \eqref{eq:targetbound} to each action gives
\[
\sup_{a\in\mathcal A}
\|\widehat g(a)-g(a)\|_2
\leq
\max_{a\in\mathcal A}
r_a\big(
\varepsilon_{\mathrm{fit},a}
+
\varepsilon_{\mathrm{target},a}
\big).
\]
This provides the stated connection between directional
response accuracy and intervention-selection error.

\paragraph{Linearized predictions.}
For the linear forecast $J_\theta(u)\delta_a$, assume
that $J^*(x)v_a$ is $L_{*,a}$-Lipschitz along the action
segment.
Then
\begin{equation}
\|J_\theta(u)\delta_a-g(a)\|_2
\leq
r_a\|[J_\theta(u)-J^*(u)]v_a\|_2
+
\frac{L_{*,a}r_a^2}{2}.
\label{eq:app_linear_bound}
\end{equation}
Indeed, adding and subtracting $J^*(u)\delta_a$
separates derivative estimation error from the
physical mean's Taylor remainder.

\paragraph{Calibrated linear predictions.}
For a scalar calibration coefficient $\alpha$, the
forecast is
\[
\widehat g_\alpha(a)=\alpha J_\theta(u)\delta_a.
\]
The same argument gives
\begin{equation}
\|\widehat g_\alpha(a)-g(a)\|_2
\leq
r_a\|[\alpha J_\theta(u)-J^*(u)]v_a\|_2
+
\frac{L_{*,a}r_a^2}{2}.
\label{eq:app_calibrated_bound}
\end{equation}
Thus, Proposition~\ref{cor:decision} applies with
\[
\eta=
\max_{a\in\mathcal A}
\left\{
r_a\|[\alpha J_\theta(u)-J^*(u)]v_a\|_2
+
\frac{L_{*,a}r_a^2}{2}
\right\}.
\]
This expression uses the error of the calibrated
derivative itself.

If desired, that derivative error can be decomposed as
\[
\|[\alpha J_\theta-J^*]v_a\|_2
\leq
\|[J_\theta-J^*]v_a\|_2
+
|\alpha-1|\,\|J_\theta v_a\|_2.
\]
This conservative decomposition separates the original
derivative error from the change introduced by calibration.
Any finite Monte Carlo error in the implemented forecast
must also be included in the action-prediction error.

\paragraph{Normalization.}
Let $d$ denote the output dimension and let $C>0$
be a normalization shared by all candidate actions
in a decision problem.
For the pixel-mean squared tracking loss
\[
L_{\mathrm{pixel}}(a)
=
\frac{\|g(a)-y\|_2^2}{d},
\]
define
\[
e_{\mathrm{track}}(a)
=
\sqrt{\frac{L_{\mathrm{pixel}}(a)}{C}}
=
\frac{\|g(a)-y\|_2}{\sqrt{dC}}.
\]
Equation~\eqref{eq:tracking_regret} therefore gives
\[
e_{\mathrm{track}}(\widehat a)
-
e_{\mathrm{track}}(a^*)
\leq
\frac{2\eta}{\sqrt{dC}}.
\]
If squared-loss regret uses the same normalization,
its bound follows by dividing
\eqref{eq:regret} by $dC$.

\section{Experimental protocol}
\label{app:protocol}

\subsection{Shared setup}

Each example consists of an initial condition,
an intervention value, and a physical outcome.
Endpoint training provides examples of the outcomes.
Response supervision additionally provides information
about how an outcome changes when the intervention
changes.

All three-seed comparisons use training seeds
32, 42, and 52.
We compute each metric separately for each trained
model, then report the mean and sample standard
deviation across seeds.


Conditional methods receive the intervention as an
input.
The unconditional DSBM reference omits this input,
so its predicted intervention response is zero.
Only the explicitly response-supervised variants
receive sensitivity targets.

The learned transport evolves over bridge time
$t\in[0,1]$.
This is distinct from the physical time used by
the simulator.
Tangents differentiate the implemented discrete
model rollout.
When evaluating a finite intervention, we use the
same model noise at the original and perturbed
parameter values, so their difference isolates
the effect of the intervention. In contrast, the two response
estimates in the conditional-mean training loss use
independent noise, as required by that objective.

\subsection{Dataset construction}

\paragraph{Conditional Gaussian.}
The coefficients in
Equation~\ref{eq:exp_gaussian} are
\[
A=
\begin{pmatrix}
0.80 & -0.30\\
0.25 & 0.65
\end{pmatrix},
\qquad
B_1=
\begin{pmatrix}
0.70\\-0.40
\end{pmatrix},
\qquad
B_2=
\begin{pmatrix}
0.25\\0.15
\end{pmatrix},
\qquad
B_3=
\begin{pmatrix}
0.55\\-0.45
\end{pmatrix}.
\]
The initial-state and additive endpoint-noise
covariances are
\[
\Sigma_0=
\begin{pmatrix}
1&0.35\\0.35&0.80
\end{pmatrix},
\qquad
\Sigma=
\begin{pmatrix}
0.25&0.08\\0.08&0.18
\end{pmatrix}.
\]

Endpoint training uses $u\in\{-1,0,1\}$.
Seen evaluation uses fresh examples at these
same parameter values.
ID interventions are sampled uniformly from
$[-1,1]$.
Near-OOD and Far-OOD use randomly signed
magnitudes drawn uniformly from $[1.15,1.50]$
and $[1.75,2.25]$, respectively.
Additional response locations are sampled
uniformly over $[-1.7,1.7]$.

The sensitivity target is the analytic derivative
\[
j^*(u)=B_1+2B_2u+\pi B_3\cos(\pi u).
\]
Because the additive noise is independent of $u$,
this derivative is the same for every noise
realization.
The response-supervision range includes Near-OOD
but does not cover Far-OOD.

\paragraph{Stochastic double well.}
We simulate the physical process from $\tau=0$
to $\tau=2$ using 200 Euler--Maruyama steps
of size $\Delta\tau=0.01$.
The noise coefficient is $0.45$, and the
intervention-dependent tilt is $0.30\tanh(u)$.

Endpoint anchors are $\{-1,0,1\}$.
ID interventions lie in $[-1,1]$;
Near-OOD uses $|u|\in[1.1,1.3]$;
Far-OOD uses $|u|\in[1.4,1.7]$.
Additional response locations cover $[-1.3,1.3]$.
Thus, Near-OOD tests generalization beyond
endpoint observations while remaining within
the response-supervision range.

For each simulator trajectory, the tangent starts
at zero and follows
\[
dR_\tau=
\left[
(1-3X_\tau^2)R_\tau+
0.30\,\operatorname{sech}^2(u)
\right]d\tau.
\]
Averaging its terminal value over independent
physical noise realizations gives the response
of the conditional mean.
This is the target used for mean-response
supervision.

\paragraph{PDEBench reaction--diffusion.}
The PDEBench simulator
\citep{takamoto2022pdebench} evolves two fields
on a $128\times128$ grid over $[-1,1]^2$.
The physical simulation runs from $\tau=0$
to $\tau=5$, with the solution recorded
at 101 time points.
The reference coefficients are
\[
D_1^0=10^{-3},
\qquad
D_2^0=5\times10^{-3},
\qquad
k^0=5\times10^{-3}.
\]
Intervention coordinates multiply these
coefficients by $2^{u_1}$, $2^{u_2}$, and
$2^{u_3}$, respectively.

Ten initial-condition seeds evaluated at the
27 anchors in $\{-1,0,1\}^3$ provide
270 endpoint-training pairs.
ID interventions lie in $[-1,1]^3$.
Near-OOD and Far-OOD use the respective shells
\[
1.1\leq\|u\|_\infty\leq1.3,
\qquad
1.4\leq\|u\|_\infty\leq1.7.
\]
Response supervision uses 108 conditions at
endpoint anchors and 216 additional conditions
over $[-1.3,1.3]^3$.

For each response condition, we draw
$z\sim\mathcal N(0,I_3)$ and set
$v=z/\|z\|_2$.
This gives an isotropic unit direction in
the three-dimensional intervention space.
The target is
\[
\widetilde j(c,u,v)=
\frac{
F(c,u+0.03v)-F(c,u-0.03v)
}{0.06},
\]
where $F$ denotes the terminal simulator field.
Both calls use the same initial condition.
The simulator is deterministic under these
fixed inputs; the resulting target describes
parameter dependence rather than physical
noise variability.

\paragraph{Stochastic Navier--Stokes.}
We generate the dataset using the SPDEBench solver
\citep{zhu2025spdebench}, with viscosity
$\nu(u)=10^{-4}2^u$ controlled by the intervention $u$.
We simulate from $\tau=0$ to $\tau=1$
with time step $\Delta\tau=10^{-3}$
and record the final vorticity field
on a $64\times64$ grid.
The fluid is driven by stochastic forcing, a random
input representing external disturbances, with
amplitude $0.005$.
The remaining SPDEBench solver settings are
\texttt{truncation}$=128$, $\alpha_Q=0.005$, and
$\kappa=10$, following the parameter definitions
in \citet{zhu2025spdebench}.

Initial fields are sampled around a shared
reference field using the Gaussian random-field
generator, with separate draws for training
and evaluation.
The model observes the initial field and
viscosity intervention, but not the realization
of the stochastic forcing.

Endpoint anchors are $\{-1,0,1\}$.
The held-out intervention sets are
\[
\begin{aligned}
\mathrm{ID}&:\{-0.75,-0.25,0.25,0.75\},\\
\mathrm{Near\mbox{-}OOD}&:\{-1.25,1.25\},\\
\mathrm{Far\mbox{-}OOD}&:\{-1.5,1.5\}.
\end{aligned}
\]
Additional response locations are sampled
uniformly over $[-1,1]$.

To construct a response target, the simulator
is evaluated at $u+0.03$ and $u-0.03$ using
the same random forcing within each pair.
The difference divided by $0.06$ is then
averaged over independent forcing realizations.
Finite-intervention evaluation uses a
log-viscosity displacement of $0.25$.
We estimate the simulator's sensitivity to the
intervention numerically and use these estimates
as response-supervision targets.

\subsection{Models and optimization}

\paragraph{Conditional and tangent bridges.}
For the Gaussian and double-well systems, the drift
network is a multilayer perceptron with three hidden
layers, each containing 128 units, and SiLU activations.
They use 30 Euler--Maruyama steps,
reference diffusion $0.5$,
learning rate $10^{-4}$,
bridge batch size 512, and
1,200 updates per matching stage.
Response updates use batch size 256.

Conditional and tangent variants share the
checkpoint after bridge-fitting iteration 3
and continue through iteration 7.
The sensitivity weights are $0.5$ for Gaussian
and $0.25$ for double well.

PDEBench uses a U-Net with base width 32,
20 rollout steps, reference diffusion $0.15$,
learning rate $2\times10^{-4}$, and bridge
batch size 8. Each standard matching stage
uses 800 gradient updates.

Tangent-SBM starts from the conditional baseline's
checkpoint after the first outer bridge-fitting
iteration and trains for one additional outer
iteration. These outer iterations each contain
multiple optimization updates.
During continuation, response supervision is applied
at every forward matching update, with sensitivity
weight $2$ and response batch size 4.

SNS uses the conditional field architecture
with reference diffusion $1.4$.
The retained conditional and tangent models
share iteration 3 and continue through
iteration 5.
Tangent-SBM uses sensitivity weight $1000$,
response batch size 8, one independent
rollout pair per response example, and
a response update every ten forward updates.

\paragraph{GSBM and TSBM adaptations.}
GSBM and TSBM receive intervention conditioning
and trajectory costs, but no sensitivity labels.
The costs express preferences over intermediate
states rather than directly supervising intervention
responses.

For the Gaussian system, both adaptations use
\[
V(x)=\frac{\beta}{2}\|x\|_2^2,
\qquad
\beta=\frac{0.1}{1.89778435}\approx0.052693,
\]
where the denominator is the training endpoint
second moment. TSBM uses the cost $\sigma^2 V$
under its diffusion-scaling convention.

For the double-well system, both use the physical
potential
\[
U(x,u)=\frac{x^4}{4}-\frac{x^2}{2}
       -0.30x\tanh(u),
\]
with weight $\beta=1$. Its negative derivative
with respect to $x$ gives the deterministic
component of the simulator drift.

For PDEBench and SNS, the cost encourages
spatial smoothness:
\[
V(x)=\frac{\beta}{2}
\sum_{(i,j)\in\mathcal E}\|x_i-x_j\|_2^2,
\qquad \beta=1,
\]
where $x$ is the normalized field and $\mathcal E$
connects neighboring grid locations.
PDEBench uses nonperiodic differences, while SNS
uses periodic differences.
TSBM applies its diffusion-dependent cost scaling.
This cost is a spatial smoothness prior, not a
residual of the governing equations.

For SNS and PDEBench, GSBM and TSBM are trained from scratch
for five forward/backward cycles, with 150
path-fitting updates and 800 drift-matching
updates per pass. Tangent-SBM also uses 800
matching updates per stage. 




\subsection{Metrics and evaluation}

Let $i$ index evaluation conditions and let
$d$ be the output dimension.
Write $\widehat m_i$ and $m_i^*$ for the
estimated model and physical mean outcomes.

\paragraph{Endpoint error.}
For Gaussian and double well, endpoint error
is root-mean-square error:
\[
E_{\mathrm{end}}=
\left(
\frac{1}{nd}
\sum_{i=1}^{n}
\|\widehat m_i-m_i^*\|_2^2
\right)^{1/2}.
\]
For the field benchmarks, we use mean
condition-wise relative error:
\[
E_{\mathrm{end}}=
\frac{1}{n}\sum_{i=1}^{n}
\frac{\|\widehat m_i-m_i^*\|_2}
{\max(\|m_i^*\|_2,\varepsilon)}.
\]
PDEBench computes this quantity in its
training-normalized field coordinates.
SNS restores physical field units before
evaluation.
The denominator floors are $10^{-8}$ for
PDEBench and $10^{-12}$ for SNS.

\paragraph{Sensitivity error.}
Sensitivity error measures how accurately the model
predicts the response to a small intervention.
At evaluation condition $i$, let $J_i^*$ denote
the reference Jacobian of the expected physical
outcome with respect to the intervention.
Let $\widehat J_i$ denote the corresponding model
estimate, computed by averaging derivatives
across model rollouts.

For Gaussian and double well, we combine errors
across conditions before normalizing:
\[
E_J=
\frac{
\left(\sum_i
\|\widehat J_i-J_i^*\|_F^2\right)^{1/2}
}{
\max\left\{
\left(\sum_i\|J_i^*\|_F^2\right)^{1/2},
\varepsilon
\right\}
}.
\]
The Frobenius norm $\|\cdot\|_F$ is the square
root of the sum of squared matrix entries.

For PDEBench and SNS, we measure the response
along an intervention direction $v_i$.
The reference and predicted responses are
$j_i^*=J_i^*v_i$ and
$\widehat j_i=\widehat J_i v_i$, respectively.
We normalize the error at each condition
and then average:
\[
E_J=
\frac{1}{n}\sum_{i=1}^{n}
\frac{\|\widehat j_i-j_i^*\|_2}
{\max\{\|j_i^*\|_2,\varepsilon\}},
\]
where $n$ is the number of evaluated
condition--direction pairs.
PDEBench uses responses in normalized field
coordinates, while SNS uses physical units.
The denominator floor prevents division by zero:
$\varepsilon=10^{-8}$ for Gaussian, double well,
and PDEBench, and $\varepsilon=10^{-12}$ for SNS.

Both definitions measure error relative to the
reference sensitivity. Because their aggregation
rules differ, we compare methods within each dataset.

\paragraph{Finite-intervention error.}
For displacement $\delta_i$, the predicted
change is
\[
\widehat D_i=
\widehat m(c_i,u_i+\delta_i)
-\widehat m(c_i,u_i).
\]
We compare this with the corresponding
physical mean change.
Errors use RMSE for the low-dimensional
systems and mean relative $\ell_2$ error
for the field systems.
The displacement magnitude is $0.25$ in
the reported intervention coordinate.
Unlike $E_J$, this metric evaluates the
effect of a nonzero parameter change.

\paragraph{Sampling budgets.}
Endpoint, sensitivity, and finite-change estimates
use 32/8/32 model samples for Gaussian,
32/16/32 for double well, 8/4/8 for PDEBench,
and 32/1,024/16 for the main SNS comparison.
For PDEBench, sensitivity and finite-change errors
are evaluated on the same eight selected conditions
per split for all methods.
For SNS, simulator reference values use 32
realizations for endpoint statistics and 16
paired simulations for sensitivities and
finite-intervention changes.

\subsection{Intervention-selection protocol}

We consider log-viscosity changes from
$\mathcal A=\{0,-0.125,0.125,-0.25,0.25\}$.
Base interventions use the four SNS ID values,
with four contexts at each value.
For each method and training seed, we estimate
a local response $\widehat j(c,u)$ from
16 model rollouts and predict an action's effect as
\[
\widehat g_\alpha(c,u,a)=\alpha a\,\widehat j(c,u).
\]
One nonnegative scale $\alpha$ is fitted per seed
on 16 separate development contexts:
\[
\alpha=
\max\left\{
0,
\frac{\sum_{i,a}
\langle a\widehat j_i,g_i^*(a)\rangle}
{\sum_{i,a}\|a\widehat j_i\|_2^2}
\right\},
\]
where $g_i^*(a)$ is the simulator-estimated mean
change and the sum uses $a=\pm0.25$.
These estimates use 64 paired simulator realizations.
We set $\alpha=0$ if the denominator vanishes
and freeze it before testing, without changing
model weights.

Evaluation uses 16 held-out contexts with two
desired changes each, generated at displacements
$\pm0.25$ using 32 paired simulator realizations.
Candidate-action outcomes are estimated independently
using 128 paired realizations.
Each of three training seeds is evaluated on
the same 32 tasks, giving 96 decisions.

For desired change $y$, define
\[
L(a)=\frac{1}{d}\|g^*(a)-y\|_2^2,
\qquad
N_y=\max\left\{\frac{1}{d}\|y\|_2^2,10^{-20}\right\},
\]
where $d$ is the number of field values.
The model selects
$\widehat a\in\arg\min_{a\in\mathcal A}
\|\widehat g_\alpha(c,u,a)-y\|_2^2$.
We evaluate its simulator-based tracking error
and normalized regret:
\[
E_{\mathrm{track}}
=\sqrt{\frac{L(\widehat a)}{N_y}},
\qquad
R=\frac{L(\widehat a)-\min_{a\in\mathcal A}L(a)}{N_y}.
\]
Tracking error measures distance from the goal;
regret measures additional squared loss relative
to the simulator-estimated best candidate action,
which defines the finite-grid oracle.

Oracle agreement is the fraction selecting
the oracle action.
Improvement is the fraction achieving lower loss
than no intervention, subject to numerical tolerance.
Exact ties follow the displayed action order,
placing zero first.
Metrics are averaged over tasks within each seed,
then summarized across seeds.
The 96 decisions reuse 32 tasks from 16 independent
contexts; the oracle uses the same candidate set
and finite-sample simulator estimates.

\paragraph{Qualitative figure.}
Figure~\ref{fig:sns} uses the first fixed
Seen case, with base intervention $u=-1$
and displacement $0.25$.
Predictions are averaged over three
training seeds.
Absolute error maps are computed per seed
and then averaged, preventing signed errors
from cancelling across seeds.
The figure illustrates one shared example,
rather than an average over all test contexts.

\section{Results across intervention regions}
\label{app:generalization}

Table~\ref{tab:generalization} extends the ID comparison
to Seen, Near-OOD, and Far-OOD.
These regions are defined relative to the parameter
values used for endpoint training.
Other evaluation settings follow
Appendix~\ref{app:protocol}.

\begin{table}[t]
\centering
\setlength{\tabcolsep}{3pt}
\caption{
Endpoint, sensitivity, and finite-change errors
across intervention regions.
Lower is better.
Entries are mean $\pm$ sample standard deviation
over three training seeds.
}
\label{tab:generalization}
\resizebox{\textwidth}{!}{%
\begin{tabular}{llccccccccc}
\toprule
& & \multicolumn{3}{c}{Seen}
  & \multicolumn{3}{c}{Near-OOD}
  & \multicolumn{3}{c}{Far-OOD}\\
\cmidrule(lr){3-5}
\cmidrule(lr){6-8}
\cmidrule(lr){9-11}
Dataset & Method
& Endpoint & $E_J$ & Finite change
& Endpoint & $E_J$ & Finite change
& Endpoint & $E_J$ & Finite change\\
\midrule
Gaussian
& DSBM
& $.559\pm.001$ & $1.000\pm.000$ & $.345\pm.000$
& $.511\pm.004$ & $1.000\pm.000$ & $.237\pm.000$
& $1.397\pm.006$ & $1.000\pm.000$ & $.512\pm.000$\\
& Conditional DSBM
& $.245\pm.008$ & $1.048\pm.007$ & $.357\pm.003$
& $.481\pm.036$ & $1.247\pm.034$ & $.249\pm.004$
& $.397\pm.028$ & $.664\pm.008$ & $.329\pm.004$\\
& GSBM
& $.264\pm.003$ & $1.051\pm.001$ & $.358\pm.001$
& $.505\pm.003$ & $1.274\pm.008$ & $.252\pm.002$
& $.462\pm.009$ & $.681\pm.008$ & $.339\pm.005$\\
& TSBM
& $.263\pm.002$ & $1.051\pm.002$ & $.358\pm.001$
& $.504\pm.005$ & $1.266\pm.006$ & $.250\pm.001$
& $.437\pm.022$ & $.672\pm.011$ & $.333\pm.006$\\
& \textbf{Tangent-SBM}
& $\mathbf{.244\pm.008}$ & $\mathbf{.035\pm.004}$ & $\mathbf{.017\pm.005}$
& $\mathbf{.245\pm.007}$ & $\mathbf{.129\pm.025}$ & $\mathbf{.021\pm.004}$
& $\mathbf{.260\pm.003}$ & $\mathbf{.137\pm.026}$ & $\mathbf{.085\pm.006}$\\
\midrule
Double well
& DSBM
& $.392\pm.004$ & $1.000\pm.000$ & $.092\pm.000$
& $.499\pm.002$ & $1.000\pm.000$ & $.042\pm.000$
& $.542\pm.005$ & $1.000\pm.000$ & $.023\pm.000$\\
& Conditional DSBM
& $.174\pm.007$ & $.458\pm.002$ & $.040\pm.001$
& $.168\pm.008$ & $1.205\pm.223$ & $.045\pm.007$
& $.167\pm.007$ & $2.549\pm.299$ & $.055\pm.006$\\
& GSBM
& $.177\pm.006$ & $.461\pm.035$ & $.040\pm.003$
& $.166\pm.004$ & $1.234\pm.170$ & $.047\pm.007$
& $.172\pm.013$ & $2.595\pm.413$ & $.056\pm.008$\\
& TSBM
& $.163\pm.006$ & $.473\pm.036$ & $.043\pm.003$
& $.160\pm.002$ & $1.342\pm.110$ & $.051\pm.005$
& $.173\pm.011$ & $2.922\pm.354$ & $.064\pm.007$\\
& Sobolev-DSBM
& $.180\pm.007$ & $.307\pm.047$ & $.027\pm.004$
& $.190\pm.011$ & $.274\pm.048$ & $.012\pm.001$
& $.203\pm.015$ & $1.193\pm.102$ & $.026\pm.003$\\
& \textbf{Tangent-SBM}
& $\mathbf{.168\pm.003}$ & $\mathbf{.271\pm.025}$ & $\mathbf{.022\pm.003}$
& $\mathbf{.166\pm.004}$ & $\mathbf{.303\pm.005}$ & $\mathbf{.013\pm.001}$
& $\mathbf{.166\pm.006}$ & $\mathbf{.984\pm.216}$ & $\mathbf{.022\pm.007}$\\
\midrule
PDEBench
& DSBM
& $.715\pm.029$ & $1.000\pm.000$ & $1.000\pm.000$
& $.740\pm.036$ & $1.000\pm.000$ & $1.000\pm.000$
& $.853\pm.032$ & $1.000\pm.000$ & $1.000\pm.000$\\
& Conditional DSBM
& $.648\pm.033$ & $.828\pm.019$ & $.833\pm.016$
& $.668\pm.020$ & $.882\pm.022$ & $.878\pm.022$
& $.729\pm.024$ & $.864\pm.065$ & $.867\pm.070$\\
& GSBM
& $.923\pm.024$ & $.801\pm.012$ & $.805\pm.010$
& $.917\pm.023$ & $.847\pm.011$ & $.845\pm.010$
& $.949\pm.027$ & $.888\pm.007$ & $.885\pm.007$\\
& TSBM
& $.906\pm.015$ & $.805\pm.036$ & $.808\pm.035$
& $.901\pm.012$ & $.847\pm.013$ & $.845\pm.012$
& $.928\pm.014$ & $.901\pm.060$ & $.900\pm.062$\\
& \textbf{Tangent-SBM}
& $\mathbf{.624\pm.072}$ & $\mathbf{.479\pm.048}$ & $\mathbf{.493\pm.049}$
& $\mathbf{.641\pm.062}$ & $\mathbf{.490\pm.048}$ & $\mathbf{.492\pm.045}$
& $\mathbf{.678\pm.066}$ & $\mathbf{.547\pm.044}$ & $\mathbf{.548\pm.045}$\\
\midrule
SNS
& DSBM
& $.301\pm.019$ & $1.000\pm.000$ & $1.000\pm.000$
& $.299\pm.016$ & $1.000\pm.000$ & $1.000\pm.000$
& $.297\pm.020$ & $1.000\pm.000$ & $1.000\pm.000$\\
& Conditional DSBM
& $.267\pm.048$ & $2.804\pm.326$ & $2.872\pm.182$
& $.266\pm.048$ & $4.754\pm.422$ & $4.051\pm.278$
& $.260\pm.047$ & $4.758\pm.339$ & $4.938\pm.257$\\
& GSBM
& $.231\pm.017$ & $3.614\pm.704$ & $3.637\pm.698$
& $.229\pm.018$ & $5.486\pm.604$ & $4.756\pm.622$
& $.230\pm.017$ & $5.295\pm.632$ & $5.498\pm.692$\\
& TSBM
& $.178\pm.003$ & $2.618\pm.022$ & $2.587\pm.031$
& $.179\pm.003$ & $4.351\pm.131$ & $3.654\pm.123$
& $.177\pm.002$ & $4.394\pm.181$ & $4.448\pm.230$\\
& Sobolev-DSBM
& $.254\pm.036$ & $2.512\pm1.048$ & $2.520\pm1.048$
& $.253\pm.035$ & $3.729\pm1.625$ & $3.092\pm1.311$
& $.247\pm.034$ & $3.556\pm1.547$ & $3.565\pm1.553$\\
& \textbf{Tangent-SBM}
& $\mathbf{.266\pm.026}$ & $\mathbf{.941\pm.072}$ & $\mathbf{1.393\pm.307}$
& $\mathbf{.265\pm.027}$ & $\mathbf{1.451\pm.146}$ & $\mathbf{1.703\pm.221}$
& $\mathbf{.259\pm.027}$ & $\mathbf{1.443\pm.204}$ & $\mathbf{1.992\pm.181}$\\
\bottomrule
\end{tabular}%
}
\end{table}

\section{Distributional evaluation}
\label{app:distributional}

A model can predict the average outcome accurately
while misrepresenting the variability of individual
outcomes. We therefore evaluate both the spread
of generated samples and their agreement with
the simulator's outcome distribution.

We examine whether adding response supervision
preserves the model's ability to reproduce
outcome variability.
Tables~\ref{tab:distribution_lowdim}
and~\ref{tab:distribution_sns} show that
Tangent-SBM retains distributional accuracy
comparable to Conditional DSBM across the
three stochastic systems.

PDEBench is deterministic given the initial
condition and intervention. Its prediction errors
are reported in Table~\ref{tab:id}; the diagnostics
below focus on the three systems with stochastic
physical outcomes.

\paragraph{Gaussian and double well.}
Table~\ref{tab:distribution_lowdim} reports
distributional errors for the two low-dimensional
systems.
For Gaussian, we compare the predicted covariance
$\widehat\Sigma$ with the reference covariance
$\Sigma$ using
\[
\frac{\|\widehat\Sigma-\Sigma\|_F}{\|\Sigma\|_F}.
\]
Covariance describes the variability of each output
component and how components vary together.
We estimate it by subtracting the estimated
conditional mean from each sampled outcome
and pooling the resulting residuals.

For double well, the final state can occupy
either of two wells.
We measure how accurately the model predicts
the probability of ending in the right well,
$\Pr(X_T>0)$, using root mean squared error
across evaluation conditions.

\begin{table}[t]
\centering
\small
\caption{ID distributional errors for Gaussian
and double well. Lower is better.
Entries are mean $\pm$ sample standard deviation
over three training seeds.}
\label{tab:distribution_lowdim}
\begin{tabular}{lcc}
\toprule
Method & Gaussian covariance error
       & Double-well probability RMSE\\
\midrule
DSBM
& $0.485\pm0.009$ & $0.170\pm0.001$\\
Conditional DSBM
& $0.330\pm0.009$ & $0.108\pm0.003$\\
GSBM
& $0.346\pm0.004$ & $0.107\pm0.001$\\
TSBM
& $0.345\pm0.003$ & $0.106\pm0.003$\\
Tangent-SBM
& $0.329\pm0.011$ & $0.105\pm0.001$\\
\bottomrule
\end{tabular}
\end{table}

Tangent-SBM gives similar distributional errors
to Conditional DSBM on both systems.
Its response improvements therefore occur
without a substantial change in these measures
of outcome variability.

\paragraph{Stochastic Navier--Stokes.}
Table~\ref{tab:distribution_sns} reports four
complementary measurements.
Here $d$ is the number of grid values in
a vorticity field.

\emph{Spread error} measures whether the model
reproduces the variability at each grid location.
We compute the standard deviation across sampled
fields at each location, then compare the model
and simulator standard-deviation fields using
relative $\ell_2$ error.

\emph{Energy distance (ED)} measures distributional
disagreement using distances between model and
simulator samples, subtracting the corresponding
within-distribution distances.
We divide its empirical value by $\sqrt d$
to account for field dimension.

\emph{Sliced Wasserstein distance (SWD)} projects
the sampled fields onto 32 unit directions.
For each direction, it compares the sorted
model and simulator projections, then averages
these differences across directions.

\emph{Enstrophy distance} compares a physical
summary of each field.
For vorticity $\omega$, enstrophy is
$\|\omega\|_2^2/(2d)$, which measures the
average squared vorticity.
We compare the model and simulator enstrophy
distributions using one-dimensional Wasserstein
distance, denoted $W_1$.

Each metric is computed separately at each
evaluation condition and then averaged.

\begin{table}[t]
\centering
\small
\setlength{\tabcolsep}{4pt}
\caption{SNS ID distributional errors.
ED is energy distance divided by $\sqrt d$;
SWD is sliced Wasserstein distance.
Enstrophy $W_1$ is displayed in units of $10^{-3}$.
Lower is better.
Entries are mean $\pm$ sample standard deviation
over three training seeds.}
\label{tab:distribution_sns}
\begin{tabular}{lcccc}
\toprule
Method & Spread error & ED & SWD
& Enstrophy $W_1$\\
\midrule
DSBM
& $.207\pm.017$ & $.0207\pm.0028$
& $.0345\pm.0019$ & $2.222\pm.439$\\
Conditional DSBM
& $.188\pm.008$ & $.0170\pm.0048$
& $.0319\pm.0045$ & $1.654\pm.123$\\
GSBM
& $.189\pm.002$ & $.0130\pm.0014$
& $.0288\pm.0014$ & $2.613\pm.806$\\
TSBM
& $.181\pm.001$ & $.0092\pm.0001$
& $.0239\pm.0005$ & $1.322\pm.113$\\
Tangent-SBM
& $.183\pm.003$ & $.0166\pm.0027$
& $.0317\pm.0025$ & $1.673\pm.490$\\
\bottomrule
\end{tabular}
\end{table}

Together, these results show that Tangent-SBM's
response improvements are accompanied by
distributional errors similar to those of
Conditional DSBM.
TSBM achieves lower errors, consistent with
its stronger endpoint fit.
Tangent-SBM's benefit is improved intervention-response
accuracy while retaining comparable distributional
fit to the conditional baseline.

\section{Response objective and rollout construction}
\label{app:objective_ablations}

\paragraph{How should an average response be supervised?}
A conditional-mean target specifies an average
over physical noise realizations.
The single-rollout control trains each
individual model response to match that average.
The two-rollout objective instead matches
the average model response without adding
the response-variance penalty identified
in Proposition~\ref{prop:cross}.
Table~\ref{tab:mean_objective} compares these
objectives on the double-well system.

\begin{table}[t]
\centering
\small
\caption{Double-well ID mean-target ablation.
Entries are mean $\pm$ sample standard deviation
over three training seeds.
The variance ratio compares conditional model
response variance with conditional simulator
response variance; it is a diagnostic rather
than a training target.}
\label{tab:mean_objective}
\begin{tabular}{lcc}
\toprule
Objective & Mean-response RMSE & Variance ratio\\
\midrule
Single-rollout squared loss
& $0.104\pm0.025$ & $0.312\pm0.055$\\
Two-rollout mean loss
& $0.066\pm0.014$ & $0.665\pm0.102$\\
\bottomrule
\end{tabular}
\end{table}

Table~\ref{tab:mean_objective} shows that the
two-rollout objective gives lower mean-response
RMSE and greater variation across trajectory
responses, consistent with removing the additional
variance penalty.
The variance ratio does not establish recovery
of physical trajectory-level responses, since
model and simulator trajectories need not have
an identified pairing.
This diagnostic reports absolute mean-response
RMSE rather than the main table's relative $E_J$.

\paragraph{Where should the response be computed?}
At prediction time, the model generates an outcome
from an initial condition without being given
the final state. We call this a free rollout.
Our response supervision measures how this
predicted outcome changes when the intervention
changes.

An alternative computes the response along a
training trajectory constrained to reach a
specified final state.
That constraint can affect the measured response,
so it may not reflect how the model responds
when predicting an endpoint on its own.

Table~\ref{tab:rollout_objective} compares these
training choices.
At evaluation, every model generates its own
endpoint without a prescribed final state,
including the model trained with
endpoint-constrained supervision.
The SNS rows also compare the single- and
two-rollout training losses.

\begin{table}[t]
\centering
\small
\caption{ID response-objective and rollout controls.
The variants describe training choices; all models
are evaluated using free rollouts.
Entries are mean $\pm$ sample standard deviation
over three training seeds. Lower errors are better.}
\label{tab:rollout_objective}
\begin{tabular}{llcc}
\toprule
System & Training variant & Endpoint error & $E_J$\\
\midrule
Double well & Free, two-rollout mean loss
& $0.136\pm0.001$ & $0.197\pm0.031$\\
& Free, single-rollout squared loss
& $0.145\pm0.006$ & $0.272\pm0.051$\\
& Endpoint-constrained
& $0.177\pm0.002$ & $0.252\pm0.020$\\
\midrule
SNS & Free, two-rollout mean loss
& $0.265\pm0.028$ & $1.002\pm0.075$\\
& Free, single-rollout squared loss
& $0.253\pm0.034$ & $2.624\pm1.103$\\
& Endpoint-constrained
& $0.252\pm0.043$ & $1.914\pm0.646$\\
\bottomrule
\end{tabular}
\end{table}

On SNS, Table~\ref{tab:rollout_objective} shows
that free-rollout supervision with the two-rollout
mean loss gives the lowest mean response error,
although the controls have slightly lower
mean endpoint errors.
For double well, free and endpoint-constrained
supervision give similar results.
These comparisons support the chosen SNS
training construction while showing that its
effect depends on the system.

\section{Target correctness and response coverage}
\label{app:target_coverage}

\paragraph{Does the supplied response information matter?}
We test whether the gains depend on correct
sensitivity information, rather than simply
adding another training loss.
The shuffled-target control assigns response
labels to the wrong training conditions.
The sign-reversed control multiplies each
response label by $-1$, asking the model to
predict the opposite response.
Both controls retain response supervision
but change the information it supplies.

Table~\ref{tab:target_correctness} compares
these controls with correct supervision.

\begin{table}[t]
\centering
\small
\caption{ID sensitivity error $E_J$ with correct,
shuffled, and sign-reversed targets.
Lower is better. Entries are mean $\pm$ sample
standard deviation over three training seeds.}
\label{tab:target_correctness}
\begin{tabular}{lccc}
\toprule
System & Correct & Shuffled & Sign-reversed\\
\midrule
Gaussian
& $.050\pm.005$ & $.858\pm.007$ & $1.817\pm.008$\\
Double well
& $.197\pm.031$ & $.357\pm.026$ & $1.027\pm.043$\\
PDEBench
& $.508\pm.045$ & $1.010\pm.008$ & $1.647\pm.060$\\
SNS
& $1.002\pm0.075$ & $1.925\pm.822$ & $2.077\pm.077$\\
\bottomrule
\end{tabular}
\end{table}

Correct targets give the lowest mean error
in every system in
Table~\ref{tab:target_correctness}.
These results support the importance of
the supplied response information:
adding a derivative loss with incorrect
targets does not reproduce the same gains.

\paragraph{How much response coverage is needed?}
Response labels can be supplied at additional
intervention settings beyond those used for
endpoint training.
We retain 25\%, 50\%, or 100\% of this additional
response-label pool to examine how coverage
affects accuracy.
For SNS, labels at the endpoint-training
settings are retained in every variant;
only the additional conditions are subsampled.

Table~\ref{tab:response_coverage} reports
the resulting sensitivity errors.
Here 100\% means the full available label
pool, not every possible intervention.

\begin{table}[t]
\centering
\small
\caption{ID sensitivity error $E_J$ when retaining
different fractions of the additional response-label
pool. Lower is better.
Entries are mean $\pm$ sample standard deviation
over three training seeds.
}
\label{tab:response_coverage}
\begin{tabular}{lccc}
\toprule
System & 25\% & 50\% & 100\%\\
\midrule
Gaussian
& $.050\pm.007$ & $.050\pm.006$ & $.050\pm.005$\\
Double well
& $.197\pm.042$ & $.198\pm.022$ & $.197\pm.031$\\
PDEBench
& $.516\pm.031$ & $.509\pm.015$ & $.508\pm.045$\\
SNS
& $2.075\pm.828$ & $1.726\pm.748$ & $1.002\pm0.075$\\
\bottomrule
\end{tabular}
\end{table}

Gaussian, double well, and PDEBench show little
change in response error across the tested
coverage levels in Table~\ref{tab:response_coverage}, suggesting that 25\% of the
additional supervision locations suffice to
maintain similar accuracy in these settings.
SNS benefits more from broader coverage; its mean sensitivity error decreases as
more response conditions are included.



\paragraph{Response direction and magnitude.}
A predicted response can have the wrong
spatial pattern, the wrong size, or both.
For SNS, we distinguish these effects using
cosine alignment and a response magnitude ratio:
\[
\operatorname{cos}(\widehat j,j^*)=
\frac{\langle\widehat j,j^*\rangle}
{\|\widehat j\|_2\|j^*\|_2},
\qquad
\rho=
\frac{\|\widehat j\|_2}{\|j^*\|_2},
\]
where $\widehat j$ is the predicted response
field and $j^*$ is the reference response field.
Denominators are floored to avoid division by zero.

Cosine alignment is $1$ for responses pointing
in the same direction, $0$ for orthogonal
responses, and $-1$ for opposite responses.
The magnitude ratio $\rho$ compares their sizes:
a value of one indicates equal magnitudes,
values above one indicate overestimation,
and values below one indicate underestimation.
For example, $\rho=0.6$ means the predicted
response has 60\% of the reference magnitude.
Matching magnitude alone does not ensure
the correct response direction.

We compute both metrics at each condition,
average across conditions, and then summarize
across training seeds.

\begin{table}[t]
\centering
\small
\caption{SNS ID response alignment and magnitude.
Higher cosine indicates better alignment;
a magnitude ratio closer to one indicates
a better match in response size.
Entries are mean $\pm$ sample standard deviation
over three training seeds.
}
\label{tab:response_geometry}
\begin{tabular}{lcc}
\toprule
Variant & Cosine alignment & Magnitude ratio\\
\midrule
Conditional DSBM
& $.121\pm.082$ & $2.547\pm.310$\\
TSBM
& $.058\pm.102$ & $2.253\pm.166$\\
Tangent-SBM, correct targets
& $.592\pm.040$ & $1.140\pm.142$\\
Tangent-SBM, shuffled targets
& $.229\pm.125$ & $1.794\pm.817$\\
Tangent-SBM, sign-reversed targets
& $-.336\pm.103$ & $1.518\pm.152$\\
\bottomrule
\end{tabular}
\end{table}

Table~\ref{tab:response_geometry} shows that
correct supervision improves alignment and
reduces magnitude overestimation relative
to Conditional DSBM and TSBM.
The sign-reversed model predicts a response
of similar size to the correctly supervised
model, but its negative cosine indicates
opposition to the reference response.
Thus, the benefit of correct targets is
not explained solely by predicting smaller
changes; their direction also matters.

\section{Simulator-budget controls}
\label{app:budget_controls}

Generating sensitivity labels requires additional
simulator runs. Could the same number of runs
improve response accuracy simply by providing
more endpoint-training examples?
We compare Tangent-SBM with two endpoint-only
controls: LocalBudget collects outcomes near
the response-supervision settings, while
SpreadBudget distributes the same simulation
budget over a broader intervention range.
Neither control uses sensitivity labels.
Table~\ref{tab:budget_comparison} reports
the resulting prediction and response errors.

\paragraph{PDEBench allocation.}
The response-label set contains 108 anchor
conditions, at intervention settings used for
endpoint training, and 216 additional conditions.
Each sensitivity label requires two simulator
runs at slightly perturbed intervention values.
The additional budget is therefore
\[
2(108+216)=648
\]
runs beyond the original 270 endpoint-training
examples.

LocalBudget adds these same perturbation outputs
to the training set as ordinary endpoint examples,
rather than using their differences as sensitivity
targets.
SpreadBudget instead uses 648 runs with fresh
initial conditions from the same generator
and interventions sampled uniformly over
$[-1.3,1.3]^3$.
Each control therefore has 918 endpoint examples.

We test two training budgets for SpreadBudget.
SpreadBudget-800 uses 800 continuation updates.
SpreadBudget-2720 increases this to 2,720 updates,
approximately preserving the number of updates
per training example:
\[
800\frac{918}{270}=2720.
\]
Table~\ref{tab:budget_comparison} includes both
variants to distinguish the effect of additional
data from that of additional optimization.

\paragraph{SNS allocation.}
Because SNS includes physical randomness,
each sensitivity label is estimated using
multiple pairs of simulator runs.
Let $n_a$ and $n_c$ be the numbers of anchor
and additional response conditions, and let
$K$ be the number of simulation pairs per label.
The additional simulation budget is
\[
B_{\mathrm{sim}}=2K(n_a+n_c),
\]
with these counts taken from the training-data
metadata.

LocalBudget generates endpoint examples at
the same positive and negative intervention
perturbations used for response supervision.
Each pair shares the same random forcing.
It matches the number and locations of the
response-label simulations, while using
new noise realizations.

SpreadBudget uses the same number of simulator
runs with fresh initial conditions and interventions
sampled uniformly over $[-1,1]$.
Both controls start from the same conditional
model checkpoint after outer bridge-fitting
iteration 3 and continue through iteration 5,
using the original normalization.

\begin{table}[t]
\centering
\small
\caption{ID simulator-budget comparisons.
LocalBudget and SpreadBudget use additional
simulations as endpoint examples; Tangent-SBM
uses them for sensitivity supervision.
Lower errors are better.
Entries are mean $\pm$ sample standard deviation
over three training seeds.
}
\label{tab:budget_comparison}
\begin{tabular}{llccc}
\toprule
System & Variant & Endpoint & $E_J$ & Finite change\\
\midrule
PDEBench & LocalBudget
& $.636\pm.071$ & $.849\pm.094$ & $.846\pm.094$\\
& SpreadBudget-800
& $.646\pm.039$ & $.875\pm.032$ & $.872\pm.036$\\
& SpreadBudget-2720
& $.661\pm.063$ & $.855\pm.109$ & $.855\pm.111$\\
& Tangent-SBM
& $.598\pm.069$ & $.508\pm.045$ & $.510\pm.046$\\
\midrule
SNS & LocalBudget
& $.272\pm.005$ & $2.929\pm.403$ & $2.517\pm.307$\\
& SpreadBudget
& $.283\pm.011$ & $3.649\pm.332$ & $3.145\pm.281$\\
& Tangent-SBM
& $.265\pm.028$ & $1.002\pm0.075$ & $1.291\pm.354$\\
\bottomrule
\end{tabular}
\end{table}

Table~\ref{tab:budget_comparison} shows that
Tangent-SBM achieves lower mean sensitivity
and finite-change errors than the tested
endpoint-only allocations on both systems.
Increasing SpreadBudget's training updates
on PDEBench does not close this gap.
These results support using the additional
simulations for explicit response supervision,
rather than solely for more endpoint examples
under the tested allocations.
The controls match simulator-run counts,
not total training time.

\section{Response supervision beyond bridges}
\label{app:alternative_predictors}

We examine whether response supervision benefits
predictors beyond stochastic bridges.
Direct terminal predictors map an initial condition
and intervention to an outcome without integrating
a learned bridge.
Deterministic field operators similarly predict
an output field directly.
For each deterministic predictor, we compare
endpoint-only training with training that
additionally uses sensitivity targets.

We also evaluate conditional flow matching,
a stochastic generative model conditioned on
the intervention, without response supervision.
This provides an additional model comparison,
rather than a paired test of response supervision.
Table~\ref{tab:alternative_predictors} reports
the results.

\begin{table}[t]
\centering
\small
\setlength{\tabcolsep}{4pt}
\caption{ID prediction errors beyond bridge models.
``Response'' denotes training with endpoint
examples and mechanistic response targets.
Lower is better.
Entries are mean $\pm$ sample standard deviation
over three training seeds.
Metrics follow Appendix~\ref{app:protocol}.}
\label{tab:alternative_predictors}
\begin{tabular}{llccc}
\toprule
System & Predictor & Endpoint & $E_J$ & Finite change\\
\midrule
Gaussian & Terminal, endpoint-only
& $.361\pm.001$ & $.865\pm.002$ & $.276\pm.001$\\
& Terminal, response
& $.023\pm.002$ & $.017\pm.001$ & $.005\pm.000$\\
\midrule
Double well & Terminal, endpoint-only
& $.125\pm.003$ & $.268\pm.008$ & $.031\pm.001$\\
& Terminal, response
& $.120\pm.001$ & $.170\pm.005$ & $.018\pm.001$\\
\midrule
PDEBench & Terminal, endpoint-only
& $.159\pm.003$ & $.416\pm.032$ & $.406\pm.017$\\
& Terminal, response
& $.130\pm.003$ & $.204\pm.011$ & $.208\pm.010$\\
& Operator, endpoint-only
& $.152\pm.002$ & $.385\pm.015$ & $.372\pm.017$\\
& Operator, response
& $.133\pm.004$ & $.260\pm.009$ & $.265\pm.009$\\
\midrule
SNS & Mean operator, endpoint-only
& $.182\pm.020$ & $70.083\pm20.495$ & $51.280\pm12.378$\\
& Mean operator, response
& $.136\pm.009$ & $3.675\pm1.666$ & $3.319\pm1.384$\\
& Conditional flow matching
& $.176\pm.028$ & $8.570\pm5.015$ & $7.568\pm4.499$\\
\bottomrule
\end{tabular}
\end{table}

\paragraph{Comparison setup.}
The PDEBench terminal variants continue from
the same initial checkpoint.
The separate PDEBench operator comparison
uses 1,600 training updates per variant.
Deterministic predictions require no averaging
over model noise; stochastic mean predictions
use the sampling budgets specified in
Appendix~\ref{app:protocol}.

\paragraph{Benefits across predictor families.}
Table~\ref{tab:alternative_predictors} shows
lower mean endpoint, sensitivity, and finite-change
errors with response supervision in every tested
deterministic predictor pair.
For example, sensitivity error decreases from
$0.865$ to $0.017$ for the Gaussian terminal
predictor and from $0.416$ to $0.204$ for
the PDEBench terminal predictor.
The SNS operator also improves substantially,
although its response errors remain above
the zero-response reference.
These comparisons show that the benefits of
mechanistic response information extend beyond
bridge learning.

\paragraph{Relation to stochastic transport.}
Direct response-supervised predictors are strong
alternatives when only mean outcomes are needed.
Several achieve lower errors than Tangent-SBM
on these mean-based metrics, particularly on
deterministic PDEBench.
That benchmark tests response learning for
high-dimensional fields rather than the need
for stochastic outcome modeling.

Tangent-SBM incorporates response information
while retaining a sampler of outcome distributions.
Its response gains over Conditional DSBM coexist
with comparable distributional fit
(Tables~\ref{tab:distribution_lowdim}
and~\ref{tab:distribution_sns}).
On SNS, its reported sensitivity and finite-change
errors are also lower than those of the tested
direct operators and conditional flow-matching
model, despite their lower endpoint errors.
The comparisons therefore distinguish accurate
mean prediction, accurate intervention responses,
and modeling outcome variability.

\section{Additional supervision studies}
\label{app:additional_ablations}

\paragraph{Forward and backward response supervision.}
The learned bridge includes a forward transport
from initial to terminal states and a backward
transport in the reverse direction.
A model that responds accurately to interventions
in one direction need not do so in the other.

For the Gaussian system, the reference conditional
distribution in the reverse direction is available
analytically, so we can also compute its intervention
sensitivity.
Table~\ref{tab:directionality} compares training
with forward-response targets, backward-response
targets, or both.
These directions refer to the learned transport,
not to reversing the physical simulator's time.

\begin{table}[t]
\centering
\small
\caption{Gaussian ID forward/backward supervision
study. Lower errors are better.
Entries are mean $\pm$ sample standard deviation
over three training seeds.}
\label{tab:directionality}
\begin{tabular}{lccc}
\toprule
Supervision & Endpoint RMSE
& Forward $E_J$ & Backward $E_J$\\
\midrule
Forward only
& $.238\pm.008$ & $.050\pm.005$ & $.905\pm.004$\\
Backward only
& $.427\pm.010$ & $.869\pm.001$ & $.101\pm.015$\\
Both directions
& $.219\pm.009$ & $.066\pm.013$ & $.108\pm.017$\\
\bottomrule
\end{tabular}
\end{table}

Table~\ref{tab:directionality} shows that
supervising one direction mainly improves
the response in that direction.
Forward-only supervision gives low forward
error but leaves a large backward error;
backward-only supervision shows the opposite
pattern.
Supervising both directions gives low response
errors in both and the lowest mean endpoint
RMSE in this study.
Forward-only training retains the lowest
mean forward response error.

These results show that response supervision
can be applied to either transport direction,
and that accuracy in the unsupervised direction
does not follow automatically.

\end{document}